\documentclass{article} 
\PassOptionsToPackage{compatV3}{fancyhdr}

\usepackage[shortlabels]{enumitem}

\usepackage[utf8]{inputenc} 
\usepackage{hyperref}       
\usepackage{url}            
\usepackage{booktabs}       
\usepackage{amsfonts}       
\usepackage{nicefrac}       
\usepackage{microtype}      
\usepackage{xcolor}         
\usepackage{fullpage}
\usepackage{authblk}
\usepackage{natbib}

\author[1]{Wilson G. Gregory}
\author[2]{George A. Kevrekidis}
\author[3]{Ben Blum-Smith}
\author[3]{Soledad Villar}

\affil[1]{Department of Statistics and Irving Institute for Cancer Dynamics, Columbia University}
\affil[2]{Division of Applied Mathematics, Brown University}
\affil[3]{Department of Applied Mathematics and Statistics, Johns Hopkins University}
    \date{}

\usepackage{parskip}

\usepackage{graphicx}
\usepackage{caption}
\usepackage{subcaption}
\usepackage{tikz}
\usepackage{tikz-cd}

\usepackage{physics}
\usepackage{amsmath}
\usepackage{amsthm}
\usepackage{amstext}
\usepackage{amssymb}
\usepackage{thmtools}
\usepackage{siunitx}
\usepackage{bm}
\usepackage{aligned-overset}
\usepackage{xfrac}

\newtheorem{theorem}{Theorem}[section]
\newtheorem*{theorem*}{Theorem}
\newtheorem{corollary}[theorem]{Corollary}
\newtheorem{lemma}[theorem]{Lemma}
\newtheorem*{lemma*}{Lemma}
\newtheorem{prop}[theorem]{Proposition}
\newtheorem*{prop*}{Proposition}
\newtheorem{definition}[theorem]{Definition}

\theoremstyle{definition}
\newtheorem{defn}{Definition}[section]
\newtheorem{remark}{Remark}[section]
\newtheorem{exmp}{Example}[section]

\usepackage[english]{babel}

\usepackage{pdfpages}

\usepackage{dsfont}
\usepackage{xcolor}
\definecolor{amber}{rgb}{1.0, 0.6, 0.0}
\usepackage{textcomp}
\usepackage{listings}
\usepackage{relsize}
\usepackage{multicol}
\usepackage{hyperref}
\hypersetup{
	colorlinks=true,
	linkcolor=black,
	filecolor=black,      
	urlcolor=black,
	citecolor=blue,
}
\usepackage[capitalise]{cleveref}
\crefname{figure}{Fig.}{Fig.}
\crefname{table}{Table}{Tables}
\crefname{equation}{Eqn.}{Eqns.}
\crefname{algocf}{Algorithm}{Algorithms}
\crefname{exmp}{Example}{Ex.}

\crefname{lemma}{Lemma}{Lemmas}
\crefname{prop}{Proposition}{Propositions}
\crefname{corollary}{Corollary}{Cor.}
\crefname{defn}{Definition}{Definitions}

\DeclareMathOperator*{\argmax}{arg\,max} 
\DeclareMathOperator*{\argmin}{arg\,min} 

\usepackage{physics}
\usepackage{url}

\usepackage{pifont}

\usepackage{bbm}
\usepackage{mathtools}
\usepackage{hhline}
\usepackage{tabularx}
\usepackage[draft, defaultcolor=brown]{changes}

\usepackage{accents}

\usepackage[ruled,vlined]{algorithm2e}

\newcommand{\R}{\mathbb{R}}
\newcommand{\bbR}{\mathbb{R}}
\newcommand{\C}{\mathbb{C}}
\newcommand{\bbZ}{\mathbb{Z}}

\newcommand{\mcT}{\mathcal{T}}

\newcommand{\Sym}{\operatorname{Sym}}
\newcommand{\Orth}{\operatorname{O}}
\newcommand{\Eucl}{\operatorname{E}}

\newcommand{\Hess}{\operatorname{Hess}}

\newcommand{\tensorname}[2]{{#1}_{(#2)}}
\newcommand{\tensor}[2]{$\tensorname{#1}{#2}$-tensor}
\newcommand{\tensors}[2]{$\tensorname{#1}{#2}$-tensors}
\NewDocumentCommand\contract{mg}{%
    \ensuremath{\iota_{#1}\IfNoValueTF{#2}{}{\qty(#2)}}%
}
\newcommand{\ZNZ}[2]{\qty(\bbZ/#1\bbZ)^{#2}}
\title{Warm-starting PDE solvers \\ with any-dimensional machine learning}

\begin{document}

\maketitle

\begin{abstract}
Any-dimensional machine learning models, such as graph neural networks (GNNs), can be naturally trained and evaluated on inputs of different sizes and dimensions. 
Inspired by the GNN transferability literature, we show mathematical conditions under which a partial differential equation (PDE) learning-based solver can be trained in small dimensions and directly applied to solve a higher dimensional PDE in a zero-shot fashion. 
These conditions are based on symmetries in both the partial differential equation and the initial data.
When the equations satisfy the symmetries but the data does not, which is the case for many PDEs arising from physics, we show that our theory gives a principled way of warm-starting low-dimensional PDE solvers for higher dimensional PDEs. 
We apply this method on the heat equation, Burgers' equation, and the compressible Navier--Stokes equations, improving the performance in both zero-shot and typical training regimes on high dimensional data.
For example, we train a surrogate model on 2D Navier--Stokes data and achieve better results on 3D test data than a baseline surrogate model trained on 3D data, while only using 12\% of the flops and 20\% of the total data size.
\end{abstract}

\section{Introduction}
In machine learning practice, it is common to use  models that are defined on a fixed set of parameters and can be evaluated on inputs of arbitrary size or dimension. These were formalized as \emph{any-dimensional machine learning models} in Eitan Levin's PhD thesis \cite{levin2026_thesis}. They are often defined on objects modulo certain symmetries, and can have uniform expressivity across input sizes or dimensions due to a phenomenon known as representation stability \citep{CHURCH2013250, yao2026any}. For example, graph neural networks operate on graphs of arbitrary size and structure, and transformer models can handle variable-length sequences, reusing the same weights across positions, where the positional encoding adapts to length. 
There is a vast literature that explains the mechanism that allows any-dimensional models to generalize across input sizes or dimensions. This started with the study of transferability or size generalization of graph neural networks \citep{ruiz2021graph, levie2021transferability, levie2024graphon, velasco2024graph, maskey2023transferability}, and it was recently generalized beyond graphs \citep{levin2025transferring, levin2026any}. 

In the literature on machine learning for PDEs, neural operators were designed with the purpose of generalizing across resolutions \citep{lu2021learning, li2020fourier, karniadakis2021physics}, but there is not much theory explaining if, how, or why a trained PDE solver can generalize to different spatial domains. The closest setting to the one considered here aggregates data from different dimensions and/or coordinate systems for PDE identification \citep{phan2025pdes, psarellis2024data}. 
Some empirical success in generalizing across dimensions has been made in the realm of foundational models for physics. For example, the Walrus model \citep{mccabe2025walruscrossdomainfoundationmodel}
approaches the problem by embedding low dimensional data in a higher dimensional space, and treating all inputs as higher dimensional.
However, this approach introduces a computational hurdle. 
Accurate PDE simulations often require data sets with high resolution \citep{szalay_turbulence_db,Yeung_Ravikumar_Uma-Vaideswaran_Dotson_Sreenivasan_Pope_Meneveau_Nichols_2025}, which is a challenge when the images scale as $\mathcal{O}(N^n)$ with side length $N$ and dimension $n$. 
Ideally, we would like a machine learning model that  exploits the efficiency benefits of a low dimensional space, and could be adapted to higher dimensions.

In this work, we use the any-dimensional machine learning framework to identify mathematical conditions under which PDE solvers can exactly generalize across spatial dimensions.
Using the intuition gained from this exercise, we develop a method that also applies when these conditions do not hold: namely, using a pretrained PDE solver in low dimensions to warm-start the learning of a PDE solver in higher dimensions. 
We empirically show that this method provides a computational advantage in more general settings.  

This work is inspired by the transferability framework for graph neural networks, and more generally the any-dimensional learning and representation stability literature (see Appendix \ref{app.lit_review} for a brief literature review). 
This framework relies on some form of symmetry: the function class must be constrained so it does not grow without bounds as the input grows.
The main assumption we make is that the PDE is symmetric with respect to spatial Euclidean transformations. 
For physics applications, Euclidean symmetry is a natural assumption because the fundamental laws of physics are homogeneous and isotropic. 
Therefore, many PDEs arising from physics satisfy Euclidean symmetries (e.g., Navier--Stokes, the Boltzmann equation, the heat equation, etc.). 
Note that for our warm-starting procedure, only the equation needs to fulfill a symmetry assumption---not the initial data, boundary conditions, or domain. 

We remark that the mathematical framework would allow us to potentially replace the Euclidean symmetry assumptions with other symmetries  (coordinate permutation, Poincar\'e symmetry, or affine symplectic symmetries, for example) with additional work. However, some form of assumption constraining the function class will always be required to be able to generalize across dimensions.

\paragraph{Our contributions}
\begin{enumerate}[(a)]
    \item We develop a framework in which a PDE can be defined across dimensions if it is expressed in terms of orthogonally invariant differential operators on jet space.  
    Examples of such operators include the Laplacian and the Dirichlet energy density.
    \item We explicitly describe conditions under which the solution of a PDE in $\R^n$ expressed in terms of $\Orth(n)$-invariant differential operators (as in (a)) can be lifted to a solution of a corresponding PDE in $\R^{n+1}$.
    \item We propose a mathematically principled method to lift learning-based PDE solvers from dimension $n$ to dimension $n+1$. 
    Under the assumptions of (b), we show that it maps PDE solutions in $n$ dimensions to PDE solutions in $n+1$ dimensions. 
    When the assumptions of (b) are not satisfied, we develop the Nepo warm-start method to provide an advantageous starting point for $(n+1)$-dimensional solvers.
    \item We demonstrate the Nepo warm-start approach with the heat equation, Burgers' equation, and the compressible Navier--Stokes equations.
    The warm-start allows, for example, a Navier--Stokes model trained only on 2D data to outperform one trained on 3D data while using 12\% of the flops and 20\% of the total data size.
\end{enumerate}

\section{Any-dimensional PDEs}\label{sec:anyd_pdes}

Consider a partial differential equation 
\begin{equation}
    \frac{\partial u}{\partial t}=D\left(x_1,\dots,x_n, u, \frac{\partial u}{\partial x_1}, \dots, \frac{\partial u}{\partial x_n}, \dots\right)
\end{equation}
characterizing the time evolution of some smooth function $u:\mathbb{R}^n\times [0,\infty) \rightarrow \mathbb{R}$. We denote $u_t:\mathbb R^n\to \mathbb R$ a point in the time evolution, but we sometimes drop the subscript $t$ for notational convenience (so that $u$ can also denote just a smooth function $\mathbb R^n\rightarrow \mathbb R$).  

A priori, it is necessary to pin down $n$ before it is even possible to write down such an equation: it controls the inputs that appear on the right side. On the other hand, the starting point for our investigation is the observation that certain important partial differential equations are written in a language that is meaningful without committing to some specific $n$. For example, the heat equation has the form
\begin{equation}
    \frac{\partial u}{\partial t} = \Delta u, 
\end{equation}
and the Laplacian operator $\Delta$ is defined on (adequately differentiable) functions on $\R^n$ for any $n$. We begin by seeking a general conceptual framework into which such examples fit. 

If the equation has partial derivatives of order at most $k$, then the function $D$ defining the partial differential equation is defined on the $k$th-order jet space 
\begin{equation}
    J^k(\R^n):=\left\{\left(x_1,\ldots, x_n, u, \frac{\partial u}{\partial x_1}, \ldots, \frac{\partial u}{\partial x_n},\frac{\partial^2 u}{\partial x_1^2},\frac{\partial^2 u}{\partial x_1\partial x_2},\dots,\frac{\partial^k u}{\partial x_n^k}\right)\right\}=\R^n\times U_{n,k}
\end{equation}
where $\R^n$ is the space domain of the partial differential equation, with coordinates $x_1,\dots,x_n$, and $U_{n,k}:=\R^{N_{n,k}}$ has coordinates corresponding to the function $u$ and all its partial derivatives up to $k$th order, of which together there are $N_{n,k}:=\binom{n+k}{k}$ in total. The space $J^k(\R^n)$ is a (trivial) fiber bundle (the {\em jet bundle}) with base $\R^n$ and fiber $U_{n,k}$. 

We want to consider differential operators that, like the Laplacian, are defined in every dimension. To do so, we consider an $n$-indexed sequence of functions $D_n: J^k(\R^n) \to \mathbb R$, giving us an $n$-indexed family of partial differential equations
\begin{equation} \label{eq.Dn}
    \frac{\partial u}{\partial t}=D_n\left(\mathbf x, u, \frac{\partial u}{\partial \mathbf x},\dots, \frac{\partial^k u}{\partial {\bf x}^k}\right)
\end{equation}
(where $\bf x$ abbreviates $x_1,\dots,x_n$ and $\partial^ju/\partial{\bf x}^j$ abbreviates the collection of partial derivatives of order $j$).
We will define a condition (the {\em compatibility condition}) under which we may regard the $D_n$ as ``the same".

\begin{minipage}{0.58\textwidth}
\begin{definition}\label{def:embed-Cinfty}
Let $\pi_n:\R^{n+1}\rightarrow \R^n$
be the projection $(x_1,\dots,x_n,x_{n+1})\mapsto (x_1,\dots,x_n)$ to the first $n$ coordinates. Given a smooth function $u\in C^\infty(\R^n)$, we can pull it back along $\pi_n$, yielding a function $\phi_n (u)\in C^\infty(\R^{n+1})$:
\begin{equation} \label{eq.phi_k=0}
\phi_n (u) := u \circ \pi_n.
\end{equation}
This defines an embedding $\phi_n:C^\infty(\R^n)\rightarrow C^\infty(\R^{n+1})$.
\end{definition}
\end{minipage}
\begin{minipage}{0.42\textwidth}
    \includegraphics[width=\linewidth]{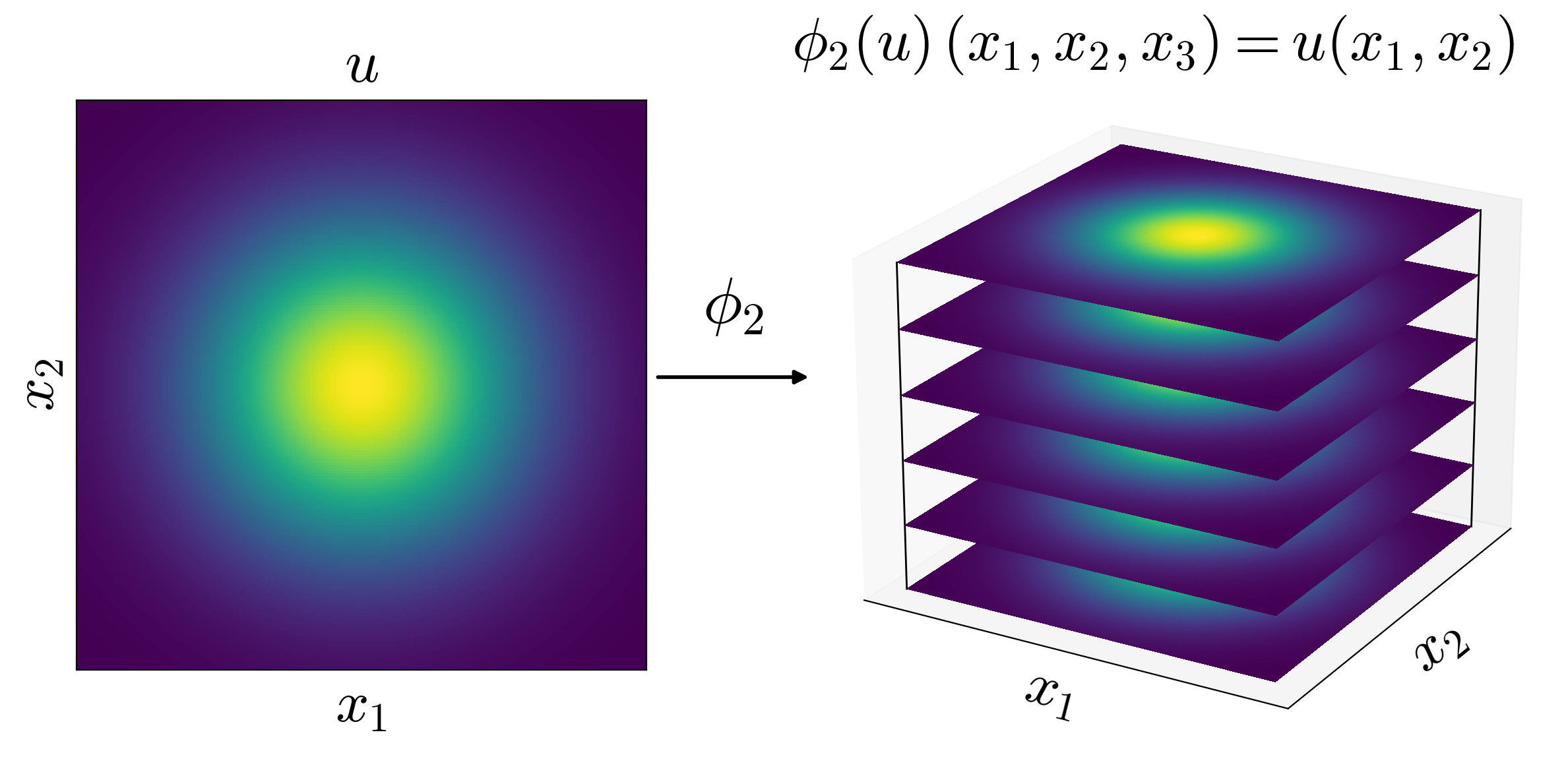}
\label{fig:duplication_diagram}
\end{minipage}

We may think of the embedding $\phi_n$ as a `duplication' embedding on smooth functions, as visualized in \Cref{def:embed-Cinfty}, where we extend a 2D scalar field as a constant in the new coordinate.

Given a smooth function $u\in C^\infty(\R^n)$, we can construct a section of the jet bundle $J^k(\R^n)$, i.e. a function in $\C^\infty(\R^n, U_{n,k})$ that sends $\mathbf x \in \R^n$ to the evaluations of $u$ and all its partial derivatives up to the $k$th order, at $\mathbf x$. 
This section is called the $k$th {\em prolongation} of $u$, and we will denote it $s_u$. 
The embedding $\phi_n$ just defined can be extended to an induced embedding, which we also call $\phi_n$, from the set $C^\infty(\R^n,U_{n,k})$ of smooth sections of $J^k(\R^n)$ into the set $C^\infty(\R^{n+1},U_{n+1,k})$ of smooth sections of $J^k(\R^{n+1})$, by extending the effect of $\phi_n$ to prolongations. 
By construction, the pulled-back function $\phi_n (u)$ does not depend on the last coordinate $x_{n+1}$, so all partial derivatives involving $x_{n+1}$ vanish, while the values of all partial derivatives involving only $x_1,\dots,x_n$ coincide with the corresponding partial derivative of $u$ evaluated at $(x_1,\dots,x_n)$. 
Therefore we make the following definition:

\begin{definition}\label{def:jet-embedding}
    We define an embedding of the fiber of $J^k(\mathbb R^n)$
    \begin{equation}\label{eq:fiber-embedding}
        \tilde \phi_n:U_{n,k}\rightarrow U_{n+1,k}
    \end{equation} 
    by zero-padding in all the coordinates corresponding to a partial derivative involving $x_{n+1}$. We define an embedding of sections of $J^k(\mathbb R^n)$ in sections of $J^k(\mathbb R^{n+1})$ by
    \begin{align}
\phi_n:C^\infty(\R^n,U_{n,k})&\rightarrow C^\infty(\R^{n+1},U_{n+1,k})\\
    s &\mapsto( \mathbf x \mapsto \tilde \phi_n(s( \pi_n(\mathbf x)))) ~,\label{eq:section-embedding}
    \end{align}
    where $s \in C^\infty(\R^{n+1},U_{n+1,k})$ is an abitrary section.
\end{definition}

Note that the $\phi_n$ in \eqref{eq:section-embedding} specializes to the $\phi_n$ in \eqref{eq.phi_k=0} when $k=0$, so the notation is justified. 
An $n$-indexed family of functions $D_n:J^k(\mathbb R^n)\rightarrow\mathbb R$, defining a sequence of differential equations as in \eqref{eq.Dn}, also induces an $n$-indexed family of maps
\begin{equation}
    \widehat D_n: C^\infty(\R^n,U_{n,k})\rightarrow C^\infty(\R^n),
\end{equation}
\begin{equation}
    \widehat D_n(s)(\mathbf x):=D_n(\mathbf x,s(\mathbf x)).
\end{equation}
We are interested in establishing a theory under which the following diagram commutes:
\begin{equation}
\label{eq.diagram}
    \begin{tikzcd}
         \arrow[r, hookrightarrow, "\dots"]
    & C^\infty(\R^{n},U_{n,k}) \arrow[r, hookrightarrow, "\phi_{n}"] \arrow[d, "\widehat D_{n}"]
    & C^\infty(\R^{n+1},U_{n+1,k}) \arrow[r, hookrightarrow, "\phi_{n+1}"] \arrow[d, "\widehat D_{n+1}"]
    & C^\infty(\R^{n+2},U_{n+2,k}) \arrow[r, hookrightarrow, "\dots"] \arrow[d, "\widehat D_{n+2}"]
    & {} \\
    \arrow[r, hookrightarrow, "\dots"]
    & C^\infty(\R_{n}) \arrow[r, hookrightarrow, "\phi_{n}"]
    & C^\infty(\R^{n+1}) \arrow[r, hookrightarrow, "\phi_{n+1}"]
    & C^\infty(\R^{n+2}) \arrow[r, hookrightarrow, "\dots"]
    & {}
    \end{tikzcd}
\end{equation}
This is a bare minimum requirement if we hope to think of the $D_n$ as being ``the same differential equation in different dimensions". But it is not enough by itself. This diagram depends on the choice of the particular orthogonal projection $\pi_n:\R^{n+1}\rightarrow \R^n$ mentioned in Definition~\ref{def:embed-Cinfty}; composing either the domain or codomain with an isometry would lead to a different projection $\pi_n':\R^{n+1}\rightarrow \R^n$, which would in turn produce a different embedding $\phi_n':C^\infty(\R^n)\hookrightarrow C^\infty(\R^{n+1})$ defined by pullback along $\pi_n'$, and different embeddings $\tilde \phi_n':U_{n,k}\hookrightarrow U_{n+1,k}$ and  $\phi_n':C^\infty(\R^{n},U_{n,k})\hookrightarrow C^\infty(\R^{n+1},U_{n+1,k})$ defined by the effect of this new pullback on prolongations. Therefore the sense in which the $D_n$'s are ``the same" is not forced to depend on a choice of coordinates, we ask that \eqref{eq.diagram} commute for these alternative $\phi_n'$'s as well.

We give a sufficient condition on the $D_n$ that guarantees this coordinate-free strengthening of \eqref{eq.diagram}:

\begin{prop}\label{prop:ortho-sym}
    If a family of functions $D_n:J^k(\mathbb R^n)\rightarrow \mathbb R$ makes \eqref{eq.diagram} commute as written (i.e., with  the family of embeddings $\phi_n$ as above), and the $D_n$ do not depend on $\mathbf x\in \R^n$ and thus are functions only of $U_{n,k}$, and if furthermore they are invariant with respect to the natural action of the orthogonal group $\Orth(n)$ on $U_{n,k}$, then \eqref{eq.diagram} would also commute with any choice of $\phi_n'$ as above in place of $\phi_n$.
\end{prop}

The hypothesis of Proposition~\ref{prop:ortho-sym} will be called the \emph{compatibility condition}.
The proof is in \Cref{ap:proof_prop_ortho_sym}.
The action of $\Orth(n)$ on $U_{n,k}$ mentioned in Proposition~\ref{prop:ortho-sym} is isomorphically its action on 
\begin{equation}
    \R\oplus \R^n \oplus \Sym^2(\R^n) \oplus \dots \oplus \Sym^k(\R^n),
\end{equation}
where $\Sym^j(\R^n)$ is the set of symmetric tensors of order $j$---this is because the collection of partial derivatives of order $j$ forms a symmetric tensor.\label{req.O-tensor-invariant}

\begin{remark}
    The compatibility condition given here is the point at which a symmetry requirement enters our theory. The proposition says that this condition guarantees that the equations $D_n$ satisfy \eqref{eq.diagram} without being committed to a particular choice of origin and orthonormal coordinates in each $\mathbb R^n$. In this sense, it is a natural condition in the context of equations modeling physical phenomena, since the laws of physics have no such commitment.
\end{remark}


\begin{remark}
    The compatibility condition only imposes a symmetry constraint on the differential equation itself, not the initial data or the boundary conditions.
\end{remark}

The compatibility condition gives us the general framework that the heat equation illustrates: any $D_n$ satisfying it can meaningfully be applied to $J^k(\R^n)$ for all sufficiently high $n$. In fact, if one fixes the degree and order of the PDE, then there is a fixed, finite basis of differential operators, independent of $n$, in which the $D_n$'s can be expressed uniformly in $n$. The Laplacian $\Delta u$ (order 2, degree 1) and the Dirichlet energy density $\|\nabla u\|^2$ (order 1, degree 2) are representative examples of such operators (see Example \ref{ex:deg-3-order-2-invars} for more examples). This can be understood in terms of {\em representation stability} for the orthogonal group, as explained in Appendix~\ref{app.examples}.

\section{Any-dimensional solution operators}\label{sec:anyd_pde_solution_operators}

We now address the question of compatibility for solution operators of differential equations.
As one might expect, even if a PDE is compatible in the sense of Proposition \ref{prop:ortho-sym}, the associated solution operator is generically not.
Nevertheless, there are families of PDEs whose solution operator does have a simple parametric dependence on the dimension of the space they are defined on; in such cases, one may hope that a lower-dimensional solution operator can be `lifted' to a higher dimensional space. For example, the heat equation 
$\pdv{u}{t} = \Delta u$ on $\R^n$ satisfies
\begin{equation}\label{eq:heat_kernel}
    u(x,t) = \frac{1}{(4\pi t)^{n/2}}\int_{\R^n}\exp(-\norm{x-y}_2^2/4t)u_0(y)dy ~,
\end{equation}
with the dimension dependence appearing purely as a scaling factor outside the integral. If such a straightforward dependence on dimension (as a parameter) is known, it would in principle allow us to transfer the solution operator between $\R^n$ subsets for varying dimension $n$. However, even if a PDE admits a solution operator in the form of a convolution, the dimension-dependent terms rarely factor outside the domain-dependent integral, and direct transferability cannot be expected. Instead, we consider a weaker equivalence framework outlined below. All the proofs of the statements of this section are in \Cref{ap:compatibility_sln_op}.

\begin{defn}
    If $D_n$ satisfies the compatibility condition of Proposition \ref{prop:ortho-sym},
    then we say that $F_n(u) = \pdv{u}{t} - \widehat{D}_n(s_u)$ is a \emph{compatible scalar PDE}, where $s_u$ is the prolongation of $u$.
\end{defn}



\begin{defn}
    Let $F_n(u)=0$ be a scalar PDE defined on a set $U\subseteq\R^n$. The associated solution operator $\mathcal{S}_{F_n}^t:(u_0)\mapsto (u_t)$ is the mapping from initial conditions to the solution of the equation at time $t$. 
\end{defn}

We then have the following result:

\begin{lemma}[Operator Symmetry]\label{lem:operator-symmetry}
    Suppose $F_n(u)=0$ is a compatible scalar PDE with associated solution operator $\mathcal{S}_{F_n}^t (u_0) \mapsto u_t$.
    If there is a uniqueness theorem for the equation $F_n=0$, then $\mathcal{S}_{F_n}^t$ is $\Eucl(n)$-equivariant, i.e. for all initial conditions $u_0$ and $g \in \Eucl(n)$,
    \begin{align}
        \mathcal{S}_{F_n}^t(g\cdot u_0) = g\cdot \mathcal{S}_{F_n}^t (u_0)~.
    \end{align}
\end{lemma}

Next, we note that any solution $u$ satisfying $F_n(u)=0$ in $\R^n$ can be embedded to the solution space of $F_{n+1}$ in $\R^{n+1}$ as follows:

\begin{lemma}[Solution Embeddings]
    \label{lem:sln_embeddings}
    For any $u\in\C^{\infty}(\R^n\times [0,\infty))$ satisfying the compatible scalar PDE $F_n(u)=0$ on $\R^n$, we have that $F_{n+1}(\phi_n(u))=0$ on $\R^{n+1}$. Furthermore, the same statement holds with $\phi_n$ replaced by arbitrary $\phi_n'$ as in Section~\ref{sec:anyd_pdes}.
\end{lemma}

For example, If $\pdv{u}{t}=\pdv[2]{u}{x_1}$ on $\R$, the embedded $\phi_n(u)$ satisfies $\pdv{u}{t}=\pdv[2]{u}{x_1}+\pdv[2]{u}{x_2}$ on $\R^2$, since the last term $\pdv[2]{u}{x_2}=0$ is identically zero.

The previous two technical lemmas can be used to show that, under this `duplication' embedding framework, the solution operators for $\mathcal{S}_{F}$ for $n$ and $n+1$ dimensions `agree', namely:

\begin{prop} \label{prop.solution-symmetry}
    Let $F_n$ be a compatible scalar PDE with a uniqueness theorem for $F_n = 0$. 
    Let $u_0$ be any initial conditions on $U\subseteq\R^n$, and let $\phi_n$ be the embedding from Definition \ref{def:embed-Cinfty}. Then, $\mathcal{S}_{F_{n+1}}^t(\phi_n(u_0))=\phi_n(\mathcal{S}_{F_n}^t(u_0))$. 
    Furthermore, the same statement holds with $\phi_n$ replaced by $\phi_n'$ as in Section~\ref{sec:anyd_pdes}.
\end{prop}

\begin{remark}[Perturbations] If a solution operator satisfies a Lipschitz condition, it is possible to control the error between solutions arising from exact- and slightly perturbed compatible embeddings. See \cref{ap:compatibility}.
\end{remark}

\begin{remark}[Vector-valued PDEs] A similar proof strategy could be followed for vector-valued evolution PDEs, where the `lift' and `projection' operations apply to each vector individually. In that case $u$ is vector-valued, and the right-hand-side $D_n$ of the PDE is equivariant with respect to $\Orth(n)$.
\end{remark}

Thus, for $(n+1)$-dimensional initial conditions that are a $\phi_n$ embedding of $n$-dimensional initial conditions, the solution operator $\mathcal{S}_{F_{n+1}}^t$ of a higher-dimensional problem can be identified with that of a lower one $\mathcal{S}_{F_n}^t$, after an appropriate projection $\pi_n$. 
This \textit{does not} give us a strong equivalence between operators in different dimensions. Instead, it suggests that the parametrization for a lower solution operator can be transferred to a higher-dimensional space \textit{while agreeing with the true high dimensional solution operator} for a subset of initial data. 
Thus, it provides a possible warm-start strategy for learning high-dimensional operators, which can be `initialized' by training in lower dimensions. We discuss this approach theoretically and implement it computationally in the following sections.

\section{Lifting PDE solvers to higher dimensions}

In this section, we extend the ideas for transferring PDE solution operators to transferring machine learning models that emulate solution operators.
In place of continuous fields, the dynamics are modeled in a grid where each pixel can have scalar, (pseudo-)vector, and (pseudo-)tensor features that discretize the corresponding fields. 
To this end, we consider the models from \cite{gi_net}, which extend equivariant convolutional neural networks to vector and tensor images. 

\subsection{Geometric convolutional neural networks}\label{subsec:ginjax}
We first define the space of (pseudo-)vectors and tensors, which will correspond to the features at each pixel (Definition \ref{def:tensors}). We then define the tensor images (Definition \ref{def:tensor_image}).

\begin{defn}[\tensors{k}{p}] \label{def:tensors}
    A \textit{\tensor{1}{p}} is an element of $\mathbb R^n$ equipped with the action of $\Orth(n)$
    \begin{align} \label{eq.action}
    g\cdot v = \det(M(g))^{\frac{1-p}{2}}\,M(g)\,v,
\end{align}
    where the parity $p$ is either $+1$ for vectors or $-1$ for pseudovectors. 
    Given $k$ \tensors{1}{p_i} denoted $v_i$, then $T:= v_1 \otimes \ldots \otimes v_k$ is a \textit{rank-1 \tensor{k}{p}}, where $p=\prod_{i=1}^k p_i$ and the action of $\Orth(n)$ is defined as 
    \begin{equation}
        g \cdot \qty(v_1 \otimes \ldots \otimes v_k) = (g \cdot v_1) \otimes \ldots \otimes (g \cdot v_k) \,.
    \end{equation}
    Thus a tensor $T$ is an element of a vector space $(\mathbb{R}^n)^{\otimes k}$, which we denote $\mcT_k\qty(\bbR^n,p)$.
    To get higher rank tensors, we can add tensors of the same order $k$ and parity $p$, and the action of $\Orth(n)$ extends linearly. 
\end{defn}

To get a discretized version of the scalar, vector, and tensor fields in the previous sections, we construct a tensor image, a grid or lattice of tensors.

\begin{defn}[\tensor{k}{p} image]\label{def:tensor_image} 
    Consider images supported in $\ZNZ{N}{n}$, the $n$-dimensional discrete torus.
    Then a \emph{\tensor{k}{p} image} is a function $A:\ZNZ{N}{n} \to \mcT_k\qty(\bbR^n, p)$.
    The vector space of \tensor{k}{p} images with side lengths $N$ is denoted $\mathcal{A}_{N,n,k,p}$.
\end{defn}

Instead of being equipped with the action of the continuous $\Eucl(n), \mathcal{A}_{N,n,k,p}$ is equipped with the action of the discrete group $G_n$, the semidirect product of discrete translations $\ZNZ{N}{n}$ and $B_n$, the group of rotations of 90 degrees and reflections.
Using tensor images, \cite{gi_net} defines geometric convolutional neural networks analogous to standard convolutional networks.
The geometric convolution, denoted $A * C$ for input image $A$ and filter $C$, is the same as standard image convolution except that the image-filter pixel multiplication is implemented as a tensor product, followed by a tensor contraction. 
These geometric convolutions characterize linear $G_d$-equivariant functions, and hence preserve the output as a tensor image, if the filters are symmetric with respect $B_d$.
See \Cref{ap:tensor_image_background} for the mathematical details, and see \Cref{fig.GINet} for an illustration.
Further layers such as point-wise nonlinearities and max pooling can be defined for tensor images (\Cref{ap:tensor_image_background}).

\begin{figure}
    \centering
    \begin{minipage}{0.37\textwidth}
    \includegraphics[width=\textwidth]{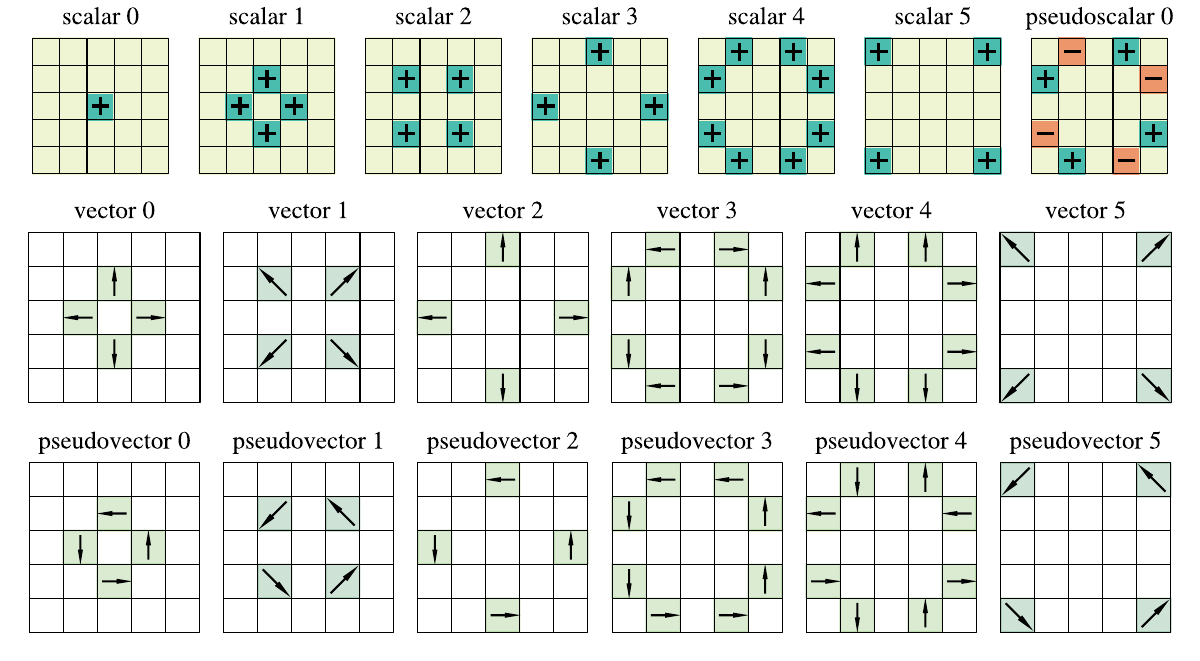}
    \end{minipage}
    \begin{minipage}{0.62\textwidth}
    \includegraphics[width=\textwidth]{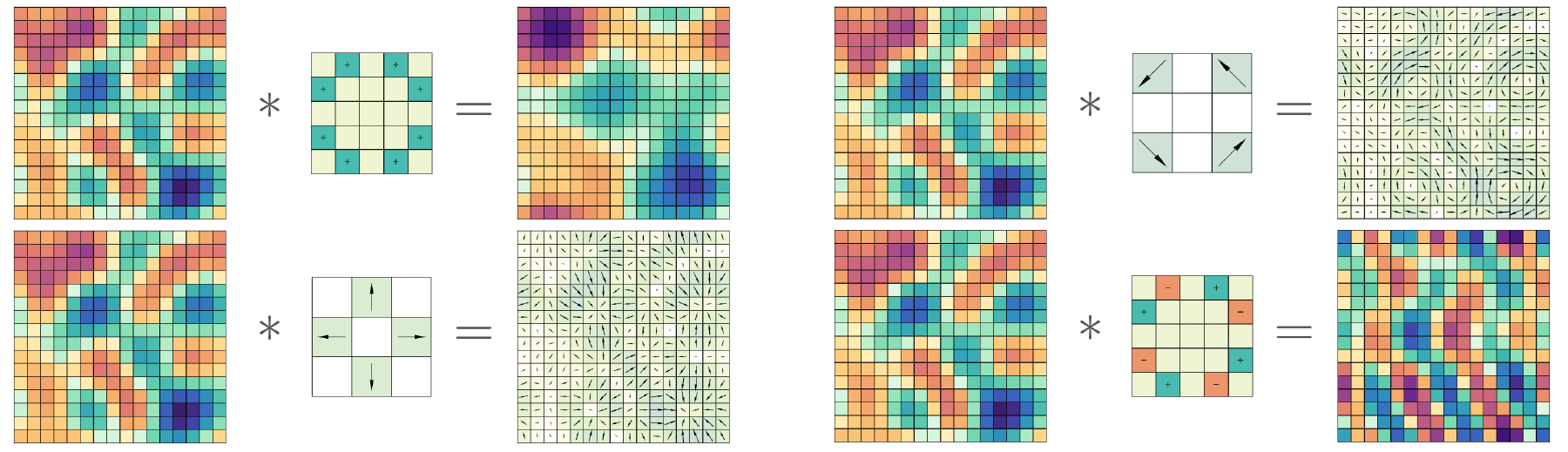}
    \end{minipage}
    \vspace{-.3cm}
    \caption{(Left) The 5x5 scalar, pseudoscalar, vector, and pseudovector symmetric filters for 2-D images. (Right) Convolution of a scalar image with different geometric filters. The convolution with a scalar filter can be seen as the discretization of a diffusion operator, and the convolution with the vector filter approximates a discretization of the gradient. In general any symmetric discretization of any coordinate-free operator (gradient, divergence, curl) could be expressed in terms of these symmetric filters.}
    \label{fig.GINet}
    \vspace{-.5cm}
\end{figure}

\subsection{Embedding of tensor images}
It is straightforward to extend the projection $\pi_n$ and embedding $\phi_n$ to the tensor image case.
\begin{defn}\label{def:tensor_projection}
    The \emph{tensor projection} $\pi_n:\mcT_k\qty(\bbR^{n+1},p) \to \mcT_k\qty(\bbR^n,p)$ is defined for all $a \in \mcT_k\qty(\bbR^{n+1},p)$ as $\qty[\pi_n(a)]_{i_1,\ldots,i_k} := \qty[a]_{i_1,\ldots,i_k}$, for $i_1,\ldots,i_k$ ranging from $1$ to $n$.
     The pseudoinverse will be the zero padding map $\tilde \phi_n:\mcT_k\qty(\bbR^n,p) \to \mcT_k\qty(\bbR^{n+1},p)$, analogous to Definition \ref{def:jet-embedding}.
\end{defn}

\begin{defn}[Tensor image embedding]\label{def:tensor_image_embedding}
    Let $A \in \mathcal{A}_{N,n,k,p}$.
    Then the \emph{tensor image embedding from $n$ to $n+1$} of $A$, denoted $\phi_n:\mathcal{A}_{N,n,k,p} \to \mathcal{A}_{N,n+1,k,p}$, is defined for all $\bar\imath_{n+1} \in \ZNZ{N}{n+1}$ as $\phi_n(A)(\bar\imath_{n+1}) = \tilde{\phi}_n\qty(A\qty(\pi_n\qty(\bar\imath_{n+1})))$.
    The pseudoinverse $\rho_n: \mathcal{A}_{N,n+1,k,p} \to \mathcal{A}_{N,n,k,p}$ takes an $n$-dimensional slice: $\rho_n(A')(\bar\imath_n) = \pi_n\qty(A'\qty(\tilde{\phi}_n(\bar\imath_n)))$, for $A' \in \mathcal{A}_{N,n+1,k,p}$ and all $\bar\imath_n \in \ZNZ{N}{n}$.
\end{defn}
We say that a sequence of functions $\{f_n\}$ are compatible with the embeddings $\{\phi_n\}$ if the following diagram commutes:
\begin{equation}\label{eq:tensor_image_compat}
    \begin{tikzcd}
        \dots \arrow[r, hookrightarrow, "\dots"] & \mathcal{A}_{N,n,k,p} \arrow[r, hookrightarrow,
        "\phi_{n}"] \arrow[d, "f_{n}"] & \mathcal{A}_{N,n+1,k,p} \arrow[r,
        hookrightarrow, "\phi_{n+1}"] \arrow[d, "f_{n+1}"] & \mathcal{A}_{N,n+2,k,p}
        \arrow[r, hookrightarrow, "\dots"] \arrow[d, "f_{n+2}"] & \dots \\ \dots
        \arrow[r, hookrightarrow, "\dots"] & \mathbb \mathcal{A}_{N,n,k',p'} \arrow[r,
        hookrightarrow, "\phi_n"] & \mathbb \mathcal{A}_{N,n+1,k',p'} \arrow[r,
        hookrightarrow, "\phi_{n+1}"] & \mathbb \mathcal{A}_{N,n+2,k',p'} \arrow[r,
        hookrightarrow, "\dots"] & \dots
    \end{tikzcd}
\end{equation}

We consider our entire neural networks as the functions $f_n$, and we enforce overall compatibility by requiring compatibility of every layer in the network.
The following theorem specifies when compatibility with respect to \eqref{eq:tensor_image_compat} holds for geometric convolutions with a particular filter.
\begin{theorem}[Geometric convolution compatibility]\label{thm:geometric_convolution_compat}
    Let $C \in \mathcal{A}_{M,n,k+k',p\,p'}$ and $C' \in \mathcal{A}_{M,n+1,k+k',p\,p'}$ be convolution filters.
    Let the functions $f_n:\mathcal{A}_{N,n,k,p} \to \mathcal{A}_{N,n,k',p'}$ and $f_{n+1}:\mathcal{A}_{N,n+1,k,p} \to \mathcal{A}_{N,n+1,k',p'}$ be defined respectively as $f_n(A) = A*C$ for all $A \in \mathcal{A}_{N,n,k,p}$ and $f_{n+1}(A')=A'*C'$ for all $A' \in \mathcal{A}_{N,n+1,k,p}$.
    Then $f_n$, $f_{n+1}$ are compatible with respect to \eqref{eq:tensor_image_compat} if
    \begin{equation}\label{eq:filter_scaling_single_c}
         C(\bar\imath_n) = \sum_{\bar\imath_{n+1} \in \pi_n^{-1}\qty(\bar\imath_n)} \pi_n\qty(C'\qty(\bar\imath_{n+1})) ~,
    \end{equation}
    for all $\bar\imath_n \in \ZNZ{N}{n}$, where $\pi_n^{-1}(\bar\imath_n)$ is the pre-image of the tensor projection to $\bar\imath_n$.
\end{theorem}

The proof is given in \Cref{ap:compat_tensor_image_functions}. 
Since tensor images form a vector space, we can think of either side of \eqref{eq:filter_scaling_single_c} as a linear combination of filters, and the same result holds (see Corollary \ref{cor:linear_combination_of_filters_compat}).
It is additionally straightforward to show that all other common layers in the network (point-wise nonlinearities, max pooling, etc.) are compatible with our embedding $\phi_n$ (\Cref{ap:compat_tensor_image_functions}).

\subsection{Algorithm for warm-starting PDEs}\label{sec:warm-start_alg}

The previous section proves conditions under which the various layers of a CNN exactly satisfy compatibility with the tensor image embedding \eqref{eq:tensor_image_compat}.
This would be the conclusion to the story if we expected all our higher dimensional images to look like the embedding: $N$ copies of the lower dimensional image stacked on top of each other.
This is unlikely to be true in practice, so we would like our lifted models to work for generic initial conditions in higher dimensions.
In this section we will provide an algorithm for transferring filters and their coefficients in a way that satisfies compatibility \eqref{eq:filter_scaling_single_c}, while simultaneously achieving better performance on generic input images.

Let $A'$ be an $(n+1)$-dimensional image, let $C'$ be an $(n+1)$-dimensional filter, and let $C$ be an $n$-dimensional filter.
We will say $C'$ and $C$ are close if the convolution $A' * C'$ followed by taking any $n$-dimensional slice is close to taking the same $n$-dimensional slice of $A'$, then convolving with $C$.
We combine $\rho$ with a random isometry to get a random slice.
This random slice projection is similar to the 2D to 3D embedding used in the Walrus model \citep{mccabe2025walruscrossdomainfoundationmodel}.

\begin{defn}[Warm-start coefficients]\label{def:warmstart_coefficients}
    Let $C_0,\ldots,C_{L_n} \in \mathcal{A}_{M,n,k,p}$ be a basis of $B_n$-invariant filters with coefficients $\alpha_0,\ldots,\alpha_{L_n}$ and let $C = \sum_{i=0}^{L_n} \alpha_i C_i$.
    Let $C'_0,\ldots,C'_{L_{n+1}} \in \mathcal{A}_{M,n+1,k,p}$ be a basis of $B_{n+1}$-invariant filters, and let $\mathcal{A}'$ be some distribution of $(n+1)$-dimensional images.
    Then the \emph{warm-start coefficients for} $\mathcal{A}'$ are the coefficients $\alpha'_0,\ldots,\alpha'_{L_{n+1}}$ with $C' = \sum_{i=0}^{L_{n+1}} \alpha'_i C'_i$ that satisfy
    \begin{equation}\label{eq:warmstart_coefficients}
        \argmin_{\alpha'_i,\ldots,\alpha'_{L_{n+1}}} \underset{h \in G_{n+1}}{\mathbb E}\qty[\underset{A' \sim \mathcal{A}'}{\mathbb E}\qty[ \mathcal{L}\qty(\qty(\rho_n\qty(h \cdot A') \ast C) - \rho_n\qty(h \cdot \qty(A' \ast C' )))]] ~,
    \end{equation}
    subject to the constraint that $C,C'$ satisfy the compatibility condition \eqref{eq:tensor_image_compat}.
    $\mathcal{L}$ is the MSE loss.
\end{defn}

Equation \ref{eq:warmstart_coefficients} can be thought of as a soft compatibility constraint, whereas \eqref{eq:tensor_image_compat} is a hard constraint.
Solving \eqref{eq:warmstart_coefficients} exactly for particular distributions of input images $\mathcal{A}'$ could be an involved problem by itself, further complicated by the fact that we would need to solve this problem for each layer of our network.
For a warm start however, we only require a reasonable starting guess from which we will further train our network end-to-end on fine-tuning data.
Thus we consider $\mathcal{A}'$ as a distribution of scalar images with simplifying assumptions, and we derive a convenient form of \eqref{eq:warmstart_coefficients} to solve.

\begin{prop}\label{prop:simple_argmin_objective}
    Let the filters and coefficients be as defined in \Cref{def:warmstart_coefficients}.
    If $\mathcal{A}'$ is the distribution of scalar images with independent identically distributed (iid) pixels with $\underset{A' \sim \mathcal{A}'}{E}\qty[A'(\bar\imath)] = 0$ and $\underset{A' \sim \mathcal{A}'}{E}\qty[A'(\bar\imath)^2] = 1$, then \eqref{eq:warmstart_coefficients} is equivalent to
    \begin{equation}\label{eq:nepo_coefficients}
        \argmin_{\alpha'_0,\ldots,\alpha'_{L_{n+1}}} \mathcal{L}\qty(\pi_n\qty(\varphi_n(C) - C')) ~,
    \end{equation}
    subject to the constraint that $C,C'$ satisfy the compatibility condition \eqref{eq:tensor_image_compat}.
\end{prop}

Note that $\pi_n$ is acting on tensor images here by projecting all the pixels, and $\varphi_n$ embeds an $n$-dimensional filter in the center slice of an $(n+1)$-dimensional filter with zeros everywhere else.
The proof and full details are given in \Cref{app.lifting_filters}.
We can now describe the full algorithm.
We call it the \emph{Nepo Warm-start} because we use our influence and connections, the $n$-dimensional data, to gain a huge advantage for our offspring, the $(n+1)$-dimensional model.
\SetKwComment{Comment}{/* }{ */}
\begin{algorithm}
    \caption{Nepo Warm-start}\label{alg:nepo_warmstart}
    \KwIn{\tensor{k}{p} filters $C_0,C_1,\ldots,C_{L_n}$ in $n$ dimensions and \tensor{k}{p} filters $C'_0,C'_1,\ldots,C'_{L_{n+1}}$ in $n+1$ dimensions for each required tensor order $k$ and parity $p$. Data $\mathcal{D}_n$ in $n$ dimensions and optionally data $\mathcal{D}_{n+1}$ in $n+1$ dimensions.}
    \begin{enumerate}[leftmargin=*]
        \item Train a ginjax model \cite{gi_net} with the $n$-dimensional \tensor{k}{p} filters $C_0,\ldots,C_{L_n}$ on $\mathcal{D}_n$.
        \item For each tensor order $k$ and parity $p$, solve \eqref{eq:nepo_coefficients} to convert the coefficients $\alpha_0,\ldots,\alpha_{L_n}$ for the $n$-dimensional filters  to the coefficients $\alpha'_0,\ldots,\alpha'_{L_{n+1}}$ for the $(n+1)$-dimensional filters.
        \item For each tensor order $k$ and parity $p$, replace the $n$-dimensional filters $C_0,\ldots,C_{L_n}$ with the $(n+1)$-dimensional filters $C'_0,\ldots,C'_{L_{n+1}}$.
        \item Optionally, continue to fine-tune the model on $\mathcal{D}_{n+1}$ data.
    \end{enumerate}
    \KwOut{The $(n+1)$-dimensional ginjax model.}
\end{algorithm}
The second step of \Cref{alg:nepo_warmstart}, computing the coefficients from \eqref{eq:nepo_coefficients}, can be done in advance analytically based on the coefficients for the $n$-dimensional filters.
These values are given in \Cref{tab:conversion_coefficients} in \Cref{app.lifting_filters}.

\section{Numerical experiments}

We demonstrate our method on three PDEs: the heat equation, Burgers' equation, and the compressible Navier--Stokes equations.
For each equation, our ultimate goal is to train a neural network to approximate the solution operator $\mathcal{S}_{F_{n+1}}^t (u_0) \mapsto u_t$. 
To do this, we use the Nepo Warm-start (\Cref{alg:nepo_warmstart}) to train a model on $\mathcal{D}_n$ data, convert the model to work with $\mathcal{D}_{n+1}$ data, and fine-tune on data sets of varying size.
We then compare each of these models with a held-out data set to a baseline model with identical architecture that was only trained on the fine-tuning data.

For the heat equation, we used $128$ 1D pre-training data points, $0,1,4,8,32,$ and $128$ 2D fine-tuning data points, and $128$ 2D test data points.
For Burgers' equation, we used $8$ 2D trajectories of pre-training data, $0,1,4$, and $8$ 3D trajectories of fine-tuning data, and $8$ 3D trajectories of test data.
Each trajectory had $50$ time steps which became $49$ input-output data points.
Finally, for the Navier--Stokes equations, we used $128$ 2D trajectories of pre-training data, $0,1,4,8,$ and $32$ 3D trajectories of fine-tuning data, and $32$ 3D trajectories of test data.
Each trajectory had $21$ time steps which became $17$ input-output data points, using $4$ time steps for input and $1$ time step for the output.
Full data and training details are in \Cref{ap:experimental_details}.

\begin{figure}[h]
    \centering
    \begin{minipage}{0.32\textwidth}
        \includegraphics[width=\textwidth]{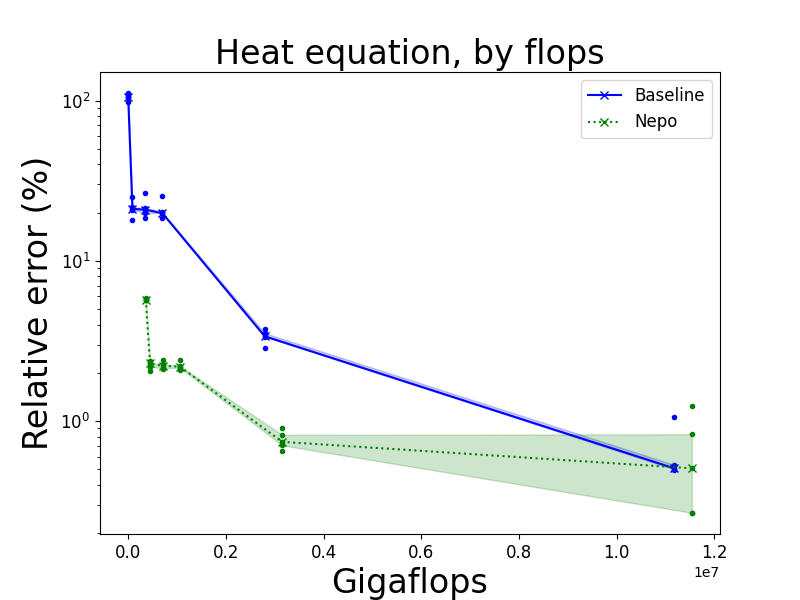}
    \end{minipage}
    \begin{minipage}{0.32\textwidth}
        \includegraphics[width=\textwidth]{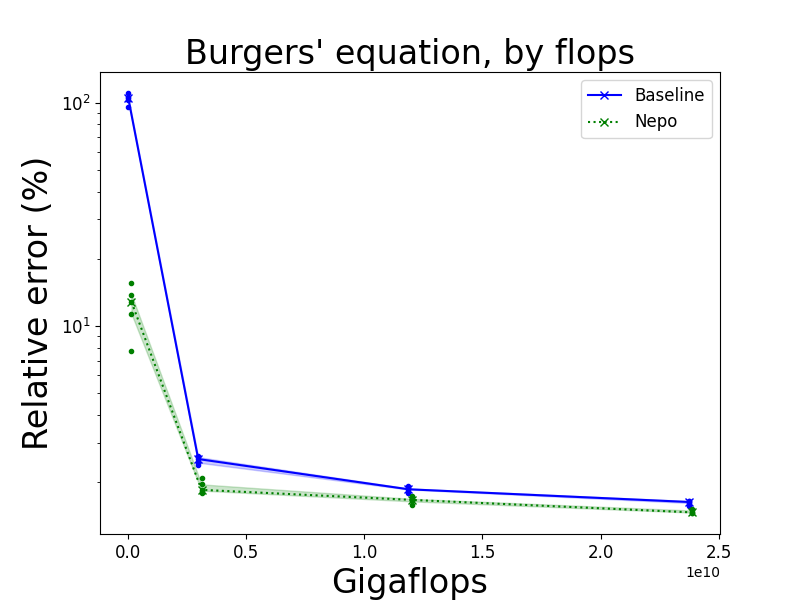}
    \end{minipage}
    \begin{minipage}{0.32\textwidth}
        \includegraphics[width=\textwidth]{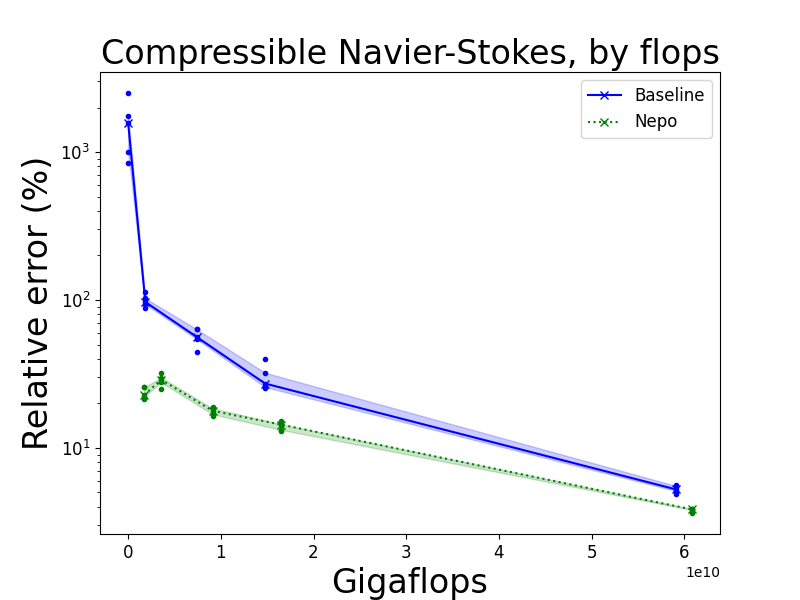}
    \end{minipage}
    \begin{minipage}{0.32\textwidth}
        \centering
        \textbf{(a)}
    \end{minipage}
    \begin{minipage}{0.32\textwidth}
        \centering
        \textbf{(b)}
    \end{minipage}
    \begin{minipage}{0.32\textwidth}
        \centering
        \textbf{(c)}
    \end{minipage}
    \caption{
    Relative error performance of \textbf{(a)} the heat equation, 
    \textbf{(b)} Burgers' equation, and \textbf{(c)} the compressible Navier--Stokes equations over varying numbers of fine tuning points, indexed by the total number of gigaflops to pre-train and fine-tune the models.
    Note the order of magnitude of the x-tick values, and the log values of the y-axis.
    We ran 5 trials, with the lines representing the median value, the shaded region representing the interquartile range, and a dot for each value.
    }
    \label{fig:warm-start_plots_relative_error}
\end{figure}

\begin{table}[h]
    \centering
    \begin{tabular}{cc|cccccc}
        PDE & Models & 0 & 1 & 4 & 8 & 32 & 128 \\
        \toprule
        Heat & Baseline & 105.3\% & 21.0\% & 20.8\% & 19.8\% & 3.3\% & \textbf{0.5\%} \\
        & Nepo & \textbf{5.7\%} & \textbf{2.2\%} & \textbf{2.2\%} & \textbf{2.1\%} & \textbf{0.7\%} & \textbf{0.5}\% \\
        \midrule
        Burgers & Baseline & 105.3\% & 2.5\% & 1.8\% & 1.6\% & Time out & Time out \\
        & Nepo & \textbf{12.8\%} & \textbf{1.8\%} & \textbf{1.6\%} & \textbf{1.4\%} & Time out & Time out \\
        \midrule
        N-S & Baseline & 1,565.7\% & 97.0\% & 56.3\% & 27.2\% & 5.2\% & Time out \\
        & Nepo & \textbf{22.8\%} & \textbf{28.9\%} & \textbf{17.8\%} & \textbf{14.3\%} & \textbf{3.8\%} & Time out \\
    \end{tabular}
    \caption{The median relative error performance over 5 trials of each PDE over varying numbers of fine tuning points.
    N-S refers to compressible Navier--Stokes.
    For the heat equation, the column labels represent the number of fine-tuning points, while for Burgers and Navier--Stokes, they refer to the number of trajectories of 17 and 49 time steps each, respectively.
    The positions in the table labeled ``Time out'' were training loops that would have taken longer than 15 hours per trial.
    With higher precision, the baseline method does better than Nepo for the heat equation by $0.001$\%.
    }
    \label{tab:relative_errors}
\end{table}

The results under the relative error metric are shown in \Cref{fig:warm-start_plots_relative_error}, and the values are in \Cref{tab:relative_errors}.
We first note that the pre-train results are more accurate than the baseline results in 12/13 cases, with a negligible difference in the last case.
The improvement is most dramatic for the most challenging PDE, the compressible Navier--Stokes equations.
The warm-start model with no additional fine-tuning performs better than the baseline model trained with 8 trajectories while using $12\%$ of the flops (\Cref{tab:flops}) and $20\%$ of the total data size.
We see a mild overfitting effect when fine-tuning the Nepo warm-start for Navier--Stokes with 1 training trajectory.
Although pre-training the model requires some additional computational time, in order to achieve comparable performance, the baseline model requires significantly more computational time.

Aside from computational time, memory is often an additional bottleneck in simulating higher dimensions.
When the data is too large to be processed all at once, it can require additional cumbersome strategies, such as splitting the spatial domain into sections to be processed individually \cite{bolton_zanna}.
For the heat equation, a single tensor image in dimension $n$ with side length $N$ requires $\mathcal{O}\qty(N^n)$ memory, while for Burgers' equation and the Navier--Stokes equations, it requires $\mathcal{O}\qty(n\,N^n)$.
Our any-dimensional method therefore provides memory benefits in addition to the computational benefits.

To confirm that \Cref{def:warmstart_coefficients} is an effective method of selecting coefficients for higher dimensional filters, we perform an ablation over other reasonable strategies in \Cref{ap:warmstart_ablation}.
Finally, additional plots and a table with the $l_2$ error are shown in \Cref{ap:additional_results} along with a table of the flops values.

\section{Discussion and limitations}

Inspired by the graph neural network transferability literature, this paper focuses on the transferability of PDE solvers across ambient dimensions. 
One key aspect for transferability in this framework is the existence of symmetries that maintain the expressive power as the dimension increases. 
In this work we assume that the right hand side of the PDE we are solving is orthogonally symmetric, but this could be changed to other symmetries (permutations or Lorentz for example) with further theoretical development.
However, some amount of symmetry in the equations will be an unavoidable assumption for this framework to apply. 
This assumption covers many applications of interest because any PDEs arising from the laws of physics will naturally have these symmetries.
Another limitation of this work is that it only focuses on flat geometries. 
A priori we do not know whether the approach could be adapted to general manifolds. 

Finally, the main limitation is that the conditions under which the solver generalizes exactly correspond to a subspace where the problem is basically a $n$-dimensional problem in $(n+1)$-dimensions. 
That is also the case in the graph neural networks setting, but they can show that difference goes to zero as the number of nodes in the graph goes to infinity. 
In our PDE application, it is impractical to consider the dimension growing to infinity, but we show that the solution in $n$ dimensions can be used as a viable warm-start for $n+1$ dimensions in practice. 
Our results numerically demonstrate that this approach gives a significant computational advantage for many problems in many settings.  

\subsection*{AI use statement}
We only used AI to produce the figure next to Definition \ref{def:embed-Cinfty}.

\subsection*{Reproducibility statement}

The details to reproduce all experiments are provided in \Cref{ap:experimental_details}.
The data sets, described in \Cref{ap:data}, are either publicly available, or generated with open source code.
Our code is open source and available at \url{https://github.com/WilsonGregory/ginjax}.

\subsection*{Acknowledgments}
WG would like to thank David W. Hogg, Mateo Diaz, and Dmitriy (Tim) Kunisky for helpful feedback on this work.
WG was partially supported by the Johns Hopkins Mathematical Institute for Data Science (MINDS) Fellowship during the development of this work. BBS was partially supported by Soledad Villar's NSF CAREER award, NSF CAREER 2339682 and NSF DMS 2031985.
SV was partially funded by NSF–Simons
Research Collaboration on the Mathematical and Scientific Foundations of Deep Learning (MoDL)
(NSF DMS 2031985), NSF CAREER 2339682, and NSF BSF 2430292.

\bibliographystyle{plain}
\bibliography{bibliography}

\newpage
\appendix

\section{Related work}
\label{app.lit_review}

\paragraph{Any-dimensional optimization and machine learning with representation stability.} A recent line of research shows how ideas from representation stability \citep{church2012homological, church2013representation} can be applied to optimization \citep{levin2023free, levin2025any} and machine learning \citep{levin2024any, levin2025transferring}. 
This framework has been recently applied to PDE identification \citep{phan2025pdes}, kernel regression \citep{diaz2025invariant}, and information theory \citep{atay2025poset}.
These ideas in particular explain why finitely-parameterized equivariant machine learning models can have uniform expressivity across growing input sizes. 
In GNNs, the notion of uniform expressivity has been explored with other techniques in~\cite{grohe2021logic, barcelo:hal-03356968, khalife2024uniform, rosenbluth2023some, boker2024fine}.
Beyond learning, representation stability has been used to study limits of a number of mathematical objects, see for example \cite{sam2017grobner, sam2015stability, sam2016gl,conca2014noetherianity, ramos2018families, alexandr2023moment}.
Closest to our work is the recent work \cite{ma2026mu}, which combines ideas of any-dimensional machine learning and $\mu$P hyperparameter transfer to warm-start large models from pretrained smaller models.

\paragraph{Equivariant machine learning.} The theory of any-dimensional machine learning applies to equivariant machine learning models, i.e., neural networks with symmetries imposed. There is a wide class of equivariant models, including the ones expressed in terms of group convolutions~\citep{cohen2016group, kondor2018generalization}, representation theory~\citep{kondor2018n, thomas2018tensor, geiger2022e3nn}, canonicalization~\citep{kaba2023equivariance}, and invariant theory~\citep{villar2021scalars,blum2023machine,villar2023dimensionless, blum2024learning}, among others. In this work, we use the equivariant models on scalars and tensor fields designed to solve PDEs from \cite{gi_net}.

\paragraph{Graph neural network transferability.} 
The GNN transferability literature explains why one can train a GNN on small graphs and generalize to larger graphs \citep{ruiz2021graph, levie2021transferability, levie2024graphon, velasco2024graph, maskey2023transferability}. These results were extended beyond graphs (to all any-dimensional models) in \cite{levin2025transferring}. The key idea is that a GNN implicitly identifies graphs of different sizes as the same object, and this induces a metric that we can use to compare graphs of different sizes under which the GNN is continuous. This allows us to derive size-generalization theoretical guarantees. In this paper, we develop a similar principle for PDE learning across dimensions. 

\section{Proof of Proposition \ref{prop:ortho-sym} }\label{ap:proof_prop_ortho_sym}

Here we prove Proposition~\ref{prop:ortho-sym}. Recall that the proposition states that if an $n$-indexed family of differential equations $D_n$:
\begin{enumerate}
    \item makes diagram~\eqref{eq.diagram} commute as written, with $\phi_n$ the standard projections $\mathbb R^{n+1}\rightarrow\mathbb R^n$, and\label{cond:commutes-with-phi}
    \item the $D_n$ do not depend on the space coordinate $\mathbf x$ and so are functions of $U_{n,k}$ alone, and\label{cond:translation-invariant}
    \item furthermore they are invariant with respect to the natural action of $\Orth(n)$ on $U_{n,k}$,\label{cond:orth-invariant}
\end{enumerate}
then they would also make diagram~\eqref{eq.diagram} commute if the $\phi_n$ were replaced by any other affine-orthogonal projections $\mathbb R^{n+1}\rightarrow\mathbb R^n$. We use this numbering for the hypotheses of the theorem in the proof. The proof itself is routine: the majority of the work is in unwinding definitions.

A notational note: it will be convenient in this proof, in order to avoid a proliferation of parentheses, to be able to write both evaluations and compositions of maps using juxtaposition, as needed. Parentheses are therefore only used when there is a danger of ambiguity, and of necessity in maps that take multiple arguments.

\begin{proof}[Proof of Proposition~\ref{prop:ortho-sym}]
We begin by unwinding definitions to translate the commutativity of diagram~\eqref{eq.diagram} into a statement about the $D_n$, in terms of the explicit maps $\pi_n$ (Definition~\ref{def:embed-Cinfty}) and $\tilde \phi_n$ (Definition~\ref{def:jet-embedding}). Prima facie, the diagram says that $\widehat D_{n+1}\phi_n = \phi_n \widehat D_n$ for all $n$. Applying both sides to an arbitrary $s\in C^\infty(\mathbb R^n,U_{n,k})$ and then evaluating at an arbitrary $\mathbf x \in \mathbb R^n$, we obtain the statement that
\begin{equation}
    (\widehat D_{n+1}(\phi_n s))(\mathbf x) = (\phi_n(\widehat D_n s))(\mathbf x)
\end{equation}
for all such $n$, $s$, $\mathbf x$. Applying the definition first of $\widehat D_{n+1}$ and then of $\phi_n:C^\infty(\mathbb R^n,U_{n,k})\rightarrow C^\infty(\mathbb R^n,U_{n,k})$ on the left side, and the definition first of $\phi_n:C^\infty(\mathbb R^n)\rightarrow C^\infty(\mathbb R^n)$ and then of $\widehat D_n$ on the right, we get
\begin{equation}\label{eq:commutativity-in-terms-of-primitives}
    D_{n+1}(\mathbf x, \tilde\phi_ns\pi_n\mathbf x) = D_n(\pi_n \mathbf x, s\pi_n \mathbf x).
\end{equation}
This equation, quantified over all $n$, $s$, and $\mathbf x$, expresses hypothesis~\ref{cond:commutes-with-phi} above. We now assume this along with hypotheses~\ref{cond:translation-invariant} and \ref{cond:orth-invariant}, and aim to conclude that \eqref{eq:commutativity-in-terms-of-primitives} would also hold with the $\pi_n$ replaced with an arbitrary affine-orthogonal projection $\mathbb R^{n+1}\rightarrow\mathbb R^n$, and a corresponding modification to $\tilde \phi_n$.

An arbitrary affine-orthogonal projection $\pi_n':\mathbb R^{n+1}\rightarrow \mathbb R^n$ has the form $\pi_n'=\pi_n g$, where $g\in \Eucl(n+1)$ is an arbitrary Euclidean-isometric self-map of $\mathbb R^{n+1}$. The action of $\Eucl(n+1)$ on $\mathbb R^{n+1}$ also induces a corresponding action on $U_{n+1,k}$, which we denote by $\tilde g$. Then the induced embedding $\tilde \phi_n':U_{n,k}\hookrightarrow U_{n+1,k}$ has the form $\tilde\phi_n' = \tilde g\tilde \phi_n$. Translations act trivially on the fiber of jet space, so the action of $\Eucl(n+1)$ on $U_{n+1,k}$ factors through $\Orth(n+1)$, and $\tilde g$ can be seen as an element of $\Orth(n+1)$ in its action on $U_{n+1,k}$.

After this setup, the proof is very short. Applying hypothesis~\ref{cond:translation-invariant} to \eqref{eq:commutativity-in-terms-of-primitives}, we can write $D_{n+1}$ and $D_n$ in terms  of the second argument only, so we have
\begin{equation}\label{eq:commutativity-under-translation-invariance}
    D_{n+1}(\tilde \phi_n s\pi_n \mathbf x) = D_n(s\pi_n \mathbf x),
\end{equation}
for all $\mathbf x\in \mathbb R^{n+1}$ and all $s$. For any $g\in \Eucl(n+1)$ (and so any $\pi_n'=\pi_n g$ and corresponding $\tilde\phi_n'=\tilde g\tilde\phi_n$), we thus have
\begin{align}
    D_{n+1}(\mathbf x,\tilde\phi_n's\pi_n'\mathbf x) &= D_{n+1}(\tilde g\tilde\phi_n s\pi_ng\mathbf x)\\
    &= D_{n+1}(\tilde\phi_ns\pi_ng\mathbf x)\\
    &= D_n(s\pi_n g\mathbf x)\\
    &= D_n(\pi_n g\mathbf x, s\pi_n g\mathbf x)\\
    &= D_n(\pi_n'\mathbf x,s\pi_n'\mathbf x),
\end{align}
where the first equality is hypothesis~\ref{cond:translation-invariant} and the definitions of $\pi_n'$ and $\tilde\phi_n'$, the second equality is hypothesis~\ref{cond:orth-invariant}, the third is from \eqref{eq:commutativity-under-translation-invariance} with $g\mathbf x$ in place of $\mathbf x$, the fourth is again hypothesis~\ref{cond:translation-invariant}, and the fifth is the definition of $\pi_n'$. This establishes \eqref{eq:commutativity-in-terms-of-primitives} with $\pi_n'$ in place of $\pi_n$ and $\tilde\phi_n'$ in place of $\tilde \phi_n$, and thus it establishes \eqref{eq.diagram} with the corresponding $\phi_n'$ in place of $\phi_n$, completing the proof.
\end{proof}

\section{Representation stability} \label{app.examples}

Representation stability is a phenomenon first noted in connection with algebraic topology \citep{church2013representation, church2012homological}. First, given a naturally $n$-indexed sequence of groups $G_n$ (such as the symmetric groups $S_n$, the general linear groups $GL(n,\R)$, or the orthogonal groups $\Orth(n)$), there is often a natural way to identify the irreducible representations of $G_n$ with some of the irreducible representations of $G_{n+1}$. Then, given a natural $n$-indexed sequence of $G_n$-representations $V_n$ and $G_n$-equivariant embeddings $V_n\rightarrow V_{n+1}$, it can come to pass under good circumstances that for sufficiently high $n$, the irreducible representations of $G_n$ appearing in $V_n$ are all the same with respect to this identification, and with the same multiplicities, for sufficiently high $n$. Such a sequence of $G_n$ representations $V_n$ is said to be {\em representation stable}. 

The canonical action of $\Orth(n)$ on $\R^n$ is an example of a representation-stable sequence of representations. From general theory, it follows that $E_{n,k}\cong \R \oplus \R^n \oplus \Sym^2(\R^n)\oplus \dots\oplus \Sym^k(\R^n)$ is representation-stable as well. In turn, it follows that the space of homogeneous degree-$d$ polynomial functions on $E_{n,k}$, which is isomorphic (as a representation of $\Orth(n)$) to $\Sym^d(E_{n,k})$, is representation-stable as well. In particular, for all sufficiently high $n$, the multiplicity of the trivial subrepresentation, i.e., the dimension of the space of $\Orth(n)$-invariants inside $\Sym^d(E_{n,k})$, does not depend on $n$. 

In fact, we can identify a basis of invariants in $\Sym^d(E_{n,k})$ for any given $d$, whose description does not depend on $n$. Any element of this basis yields a sequence of $D_n$'s satisfying the compatibility condition, and every such sequence can be described in terms of this basis.

The Laplacian $\Delta u$ is an example of such an invariant. Other examples include $u$ itself and the squared norm of the gradient of $u$ (also known as the Dirichlet energy density of $u$). We illustrate further by describing all of the $\Orth(n)$-invariants on order $2$ jet space that are of degree 1, 2, or 3 as polynomials in $u$ and its derivatives:

\begin{exmp}\label{ex:deg-3-order-2-invars}
    Given a smooth function $u:\R^n\rightarrow \R$, view the gradient $\nabla u$ as a column vector, and denote by $\Hess u$ the Hessian matrix of $u$. Then every nonconstant $\Orth(n)$-invariant polynomial of degree at most 3 in $u$ and its first and second derivatives, is in fact a polynomial in the following:
    \begin{itemize}
        \item $u$ itself
        \item $\Delta u$, the Laplacian
        \item $\nabla u\cdot \nabla u$, the Dirichlet energy density
        \item $e_2(\Hess u)$, the coefficient of the $t^2$ term in the characteristic polynomial $\det(I-t\Hess u)$ of the Hessian
        \item $\nabla u^\top (\Hess u) \nabla u$, the norm of the gradient with respect to the symmetric bilinear form defined by the Hessian
        \item $e_3(\Hess u)$, the coefficient of the $t^3$ term in the characteristic polynomial of the Hessian
    \end{itemize}
\end{exmp}

The important takeaway is that this list of invariants does not depend on $n$ (although some of the list items are zero for very low $n$). Note that the invariants listed here satisfy the compatibility requirement: for any invariant on the list, its instantiation for any $n$ is equal to its instantiation for $n+1$ upon zeroing out the derivatives that depend on $x_{n+1}$, so the diagram in \eqref{eq.diagram} commutes; Proposition~\ref{prop:ortho-sym} then implies it would also commute beginning from a different sequence of orthogonal projections $\pi_n'$.\\

\begin{remark}
    The determinant of the Hessian is an invariant for any $n$. However, unlike the invariants listed in Example~\ref{ex:deg-3-order-2-invars}, it is not naturally identified by the above theory across $n$-values, as it does not satisfy the compatibility condition---in particular, \eqref{eq.diagram} does not commute for this choice $D_n(u):=\det(\Hess u)$ (since its degree depends on $n$). What is true is that for any fixed $j$, the coefficient of the degree $j$ term of the characteristic polynomial, $D_n(u):=e_j(\Hess u)$, satisfies compatibility across values of $n$. When $n<j$, $e_j(\Hess u)$ is zero;  when $n=j$ it is equal to the determinant.
\end{remark}

\section{Tensor Image and Geometric Convolution Background}\label{ap:tensor_image_background}

In this section we will introduce the key concepts of tensor images and geometric convolution.
For narrative flow we will include the definitions for some terms previously defined in \Cref{subsec:ginjax}.

We first define the space of (pseudo-)vectors and tensors, which will correspond to the features at each pixel (Definition \ref{def:tensors_appendix}). We then define the geometric images (Definition \ref{def:tensor_image_appendix}).

\begin{defn}[\tensors{k}{p}] \label{def:tensors_appendix}
    A \textit{\tensor{1}{p}} is an element of $\mathbb R^n$ equipped with the action of $\Orth(n)$
    \begin{align} \label{eq.action_appendix}
    g\cdot v = \det(M(g))^{\frac{1-p}{2}}\,M(g)\,v,
\end{align}
    where the parity $p$ is either $+1$ for vectors or $-1$ for pseudovectors. 
    Given $k$ \tensors{1}{p_i} denoted $v_i$, then $T:= v_1 \otimes \ldots \otimes v_k$ is a \textit{rank-1 \tensor{k}{p}}, where $p=\prod_{i=1}^k p_i$ and the action of $\Orth(n)$ is defined as 
    \begin{equation}
        g \cdot \qty(v_1 \otimes \ldots \otimes v_k) = (g \cdot v_1) \otimes \ldots \otimes (g \cdot v_k) \,.
    \end{equation}
    Thus a tensor $T$ is an element of a vector space $(\mathbb{R}^n)^{\otimes k}$, which we denote $\mcT_k\qty(\bbR^n,p)$.
    To get higher rank tensors, we can add tensors of the same order $k$ and parity $p$, and the action of $\Orth(n)$ extends linearly. 
\end{defn}

The space $\mcT_k\qty(\bbR^n,p)$ naturally has addition and scalar multiplication, but also it has tensor products and tensor contractions.
The tensor product takes a \tensor{k}{p} and a \tensor{k'}{p'} and outputs a \tensor{(k+k')}{p\,p'} whose components are defined as:
\begin{equation}
    [a \otimes b]_{i_1, \ldots, i_{k+k'}} = [a]_{i_1, \ldots i_k} [b]_{i_{k+1},\ldots,i_{k+k'}} ~.
\end{equation}
The tensor contraction $\iota_k$ can be though of as analogous to the matrix trace where axes $i_1,\ldots,i_k$ are paired with $i_{k+1},\ldots,i_{2k}$ and summed over where those axis indices are equal.
Written in Einstein summation notation, the $k$-contraction of a \tensor{(2k+k')}{p} is defined as:
\begin{equation}
    \qty[\contract{k}{a}]_{i_{2k+1},\ldots,i_{2k+k'}} = \qty[a]_{i_1,\ldots,i_{2k+k'}} ~.
\end{equation}
To create a geometric image, we align tensors into a grid or lattice structure.

\begin{defn}[\tensor{k}{p} image]\label{def:tensor_image_appendix} 
    Consider images supported in $\ZNZ{N}{n}$, the $n$-dimensional discrete torus.
    Then a \emph{\tensor{k}{p} image} is a function $A:\ZNZ{N}{n} \to \mcT_k\qty(\bbR^n, p)$.
    The set of \tensor{k}{p} images with side lengths $N$ is denoted $\mathcal{A}_{N,n,k,p}$.
\end{defn}

For the four tensor operations described above---addition, scalar multiplication, tensor product, and tensor contraction---we have analogous definitions for the geometric image where the operation is performed pixel-wise. 
For example, if $A,B \in \mathcal{A}_{N,n,k,p}$, then for all $\bar\imath \in \ZNZ{N}{n}$ we have $(A+B)(\bar\imath) = A(\bar\imath) + B(\bar\imath)$. 

In this space, the geometric convolutional neural networks are analogous to standard convolutional networks, except that the image-filter pixel multiplication is implemented as tensor products, followed by contractions. 
For example, if $n=2$ (i.e. 2D images) the filters are symmetric with respect to 90 degree rotations, and either symmetric ($p=+1$) or anti-symmetric ($p=-1$) with respect to reflections. 
For general dimension $n$, this is the hyperoctahedral group, denoted $B_n$. 

\begin{defn}[Geometric convolution]\label{def:convolution}
    Let $N,M$ be side lengths where $M=2m+1$ is odd, let $A \in \mathcal{A}_{N,n,k,p}$ be the input image, and let $C \in \mathcal{A}_{M,n,k+k',p\,p'}$ be the convolution filter.
    Then the \emph{geometric convolution} is a tensor image $(A \ast C) \in \mathcal{A}_{N,n,k',p'}$ such that:
    \begin{equation}\label{eq:convolution}
        (A\ast C)(\bar\imath) = \sum_{\bar a\in \qty(\bbZ/M \bbZ)^n} \contract{k}{A\qty(\bar\imath - c\qty(\bar a))\otimes C(\bar a)} ~,
    \end{equation}
    where $c:\qty(\bbZ/M \bbZ)^n \to \bbZ^n$ is the pixel centering function $\qty[c(\bar a)]_i = \qty[\bar a]_i - \frac{M-1}{2}$.
    Note that the pixel centering is not on the torus, but $\bar\imath - c(\bar a)$ is on the torus of $\ZNZ{N}{n}$.
\end{defn}

Geometric convolutions characterize linear $G_d$-equivariant functions \cite{gi_net}, but we also need nonlinear functions to build expressive networks.
For this task, we will use the Vector neurons nonlinearity, originally from \cite{deng2021vectorneuronsgeneralframework} and extended to tensors in \cite{gi_net}.

\begin{defn}[Vector Neuron nonlinearity for tensors]\label{def:vnn_for_tensors}
    Let $A_i$ for $i=1,\ldots,r$ be input \tensor{k}{p} image channels.
    Let $\alpha_i \in \bbR$ be learned scalar parameters, let $B=\sum_{i=1}^{r} \alpha_i \, A_i$, and let $\hat{B} = \frac{B}{\norm{B}_2}$, where $\norm{\cdot}_2$ is the typical tensor norm, also known as the Frobenius norm.
    The projection of $A_i$ to $\hat{B}$ is $A_{i,\parallel} := \contract{k}{A_i \otimes \hat{B}}\hat{B}$.
    Finally, let $\tilde{\sigma}$ be a scalar nonlinearity defined on images pixel-wise.
    Then the \emph{Vector Neuron nonlinearity} $\sigma: \qty(\mathcal{A}_{N,n,k,p})^r \to \qty(\mathcal{A}_{N,n,k,p})^r$ for the $i^{th}$ output $\sigma_i$ is defined:
    \begin{equation}\label{eq:nonlinearity}
        \sigma_i\qty(\qty(A_j)_{j=1}^r) = A_i - A_{i,\parallel} + \frac{\tilde{\sigma}\qty(\contract{k}{A_i \otimes \hat{B}})}{\qty|\contract{k}{A_i \otimes \hat{B}}|} A_{i,\parallel}
    \end{equation}
\end{defn}

Additionally, we can define pooling layers such as max pooling, where the max is calculated using the $\ell_2$ norm of a tensor as follows:

\begin{defn}[max\,pool]\label{def:max_pool}
    Let $A \in \mathcal{A}_{N,n,k,p}$ and let $b$ be a positive integer that divides $N$. 
    Then $\text{max\,pool}$ is a function whose output is $\mathcal{A}_{N/b,n,k,p}$ defined for each pixel index $\bar\imath$:
    \begin{equation}\label{eq:max_pool}
        \text{max\,pool}(A,b)(\bar\imath) = A\qty(b\,\bar\imath + \argmax_{\bar a \in \ZNZ{b}{n}}\norm{A(b\,\bar\imath + \bar a)}_2) 
    \end{equation}
\end{defn}

\section{Compatibility}\label{ap:compatibility}

\subsection{Compatibility of Solution Operators}
\label{ap:compatibility_sln_op}


\begin{lemma*}[Restatement of \ref{lem:operator-symmetry}]
    Suppose $F_n(u)=0$ is a compatible scalar PDE with associated solution operator $\mathcal{S}_{F_n}^t (u_0) \mapsto u_t$.
    If there is a uniqueness theorem for the equation $F_n=0$, then $\mathcal{S}_{F_n}^t$ is $\Eucl(n)$-equivariant, i.e. for all initial conditions $u_0$ and $g \in \Eucl(n)$,
    \begin{align}
        \mathcal{S}_{F_n}^t(g\cdot u_0) = g\cdot \mathcal{S}_{F_n}^t (u_0)~.
    \end{align}
\end{lemma*}

\begin{proof}
    Let $g\in \Eucl(n)$. Let $u_t$ be the evaluation of the function $u$ at time $t$. 
    Then consider $\mathcal{S}_{F_n}^t(g\cdot u_0)$. Because $\mathcal{S}_{F_n}^t$ is a solution operator, this is the evaluation at time $t$ of a solution to $F_n=0$ with initial condition $g\cdot u_0$. 
    Meanwhile, consider $g\cdot \mathcal{S}_{F_n}^t(u_0)$. 
    Because $F_n$ is $\Eucl(n)$-invariant 
    and $F_n(\mathcal{S}_{F_n}^t(u_0))=F_n(u_t)=0$, we have 
    \begin{equation}
        0 = F_n(u_t) = F_n(g \cdot u_t) = F_n(g \cdot \mathcal{S}_{F_n}^t(u_0)) ~.
    \end{equation}
    Thus, $g \cdot \mathcal{S}_{n}^t(u_0)$ is also a solution to $F_n=0$, and its initial condition is $\sigma\cdot G_0(u_0) = g\cdot u_0$. Because we assumed a uniqueness theorem for the equation $F_n=0$, we conclude that
    \begin{equation}
        \mathcal{S}_{F_n}^t(g \cdot u_0) = g\cdot \mathcal{S}_{F_n}^t(u_0),
    \end{equation}
    so $\mathcal{S}_{F_n}^t$ is $\Eucl(n)$-equivariant.
\end{proof}

\begin{lemma*}[Restatement of \ref{lem:sln_embeddings}]
    For any $u\in\C^{\infty}(\R^n\times [0,\infty))$ satisfying the compatible scalar PDE $F_n(u)=0$ on $\R^n$, we have that $F_{n+1}(\phi_n(u))=0$ on $\R^{n+1}$. Furthermore, the same statement holds with $\phi_n$ replaced by arbitrary $\phi_n'$ as in Section~\ref{sec:anyd_pdes}.
\end{lemma*}

\begin{proof}
Compatibility implies that \eqref{eq.diagram} commutes. Thus, $\widehat D_{n+1}\circ \phi_n = \phi_n\circ \widehat D_n$. 
Applying both sides of the equality $\widehat D_{n+1}\circ \phi_n = \phi_n\circ \widehat D_n$ to the prolongation $s_u$ of the time-evolving function $u$, and then evaluating at a point $(\mathbf x,t)$ in space and time, we get 
\begin{equation}
     \widehat{D}_{n+1}(\phi_n(s_u))(\mathbf x,t) = \phi_n( \widehat{D}_n(s_u))(\mathbf x,t) =  \widehat{D}_n(s_u)(\pi_n(\mathbf x),t).
\end{equation}
Also, $\frac{d(\phi_n (u))}{dt}(\mathbf x,t) = \frac{du}{dt}(\pi_n (\mathbf x),t)$ by the chain rule. Since 
\begin{equation}
\left(\frac{du}{dt}- \widehat{D}_n(s_u)\right)(\pi_n (\mathbf x),t) = F_n(u)(\pi_n (\mathbf x),t) = 0
\end{equation}
by assumption, it follows that 
\begin{equation}
0 = F_{n+1}(\phi_n (u))(\mathbf x,t) = \left(\frac{d(\phi_n (u))}{dt} -  \widehat{D}_{n+1}(\phi_n (s_u))\right)(\mathbf x,t)
\end{equation}
as well. This proves the first statement. The statement for $\phi_n'$ follows from Proposition~\ref{prop:ortho-sym}.
\end{proof}

\begin{prop*}[Restatement of \ref{prop.solution-symmetry}]
    Let $F_n$ be a compatible scalar PDE with a uniqueness theorem for $F_n = 0$. 
    Let $u_0$ be any initial conditions on $U\subseteq\R^n$, and let $\phi_n$ be the embedding from Definition \ref{def:embed-Cinfty}. Then, $\mathcal{S}_{F_{n+1}}^t(\phi_n(u_0))=\phi_n(\mathcal{S}_{F_n}^t(u_0))$. 
    Furthermore, the same statement holds with $\phi_n$ replaced by $\phi_n'$ as in Section~\ref{sec:anyd_pdes}.
\end{prop*}
\begin{proof}
    The function $S_{F_n}^t(u_0)$ is a solution to $F_n=0$ with initial condition $u_0$, by definition of the solution operator. By \cref{lem:sln_embeddings}, $\phi_n(S_{F_n}^t(u_0))$ is thus a solution to $F_{n+1}=0$, with initial condition $\phi_n(u_0)$.
    Again by definition of the solution operator, $\mathcal{S}_{F_{n+1}}^t(\phi_n(u_0))$ is a solution to the same equation with the same initial condition. By uniqueness, they are equal.
    The statement for $\phi_n'$ then follows from Lemma~\ref{lem:operator-symmetry}.
\end{proof}

\textbf{On Remark 3.1} An evolutionary PDE induces a solution operator $S_{F_n}^t:X\to Y$ where $X_n$ and $Y_n$ are Banach spaces of initial conditions and solutions at time $t$ over the corresponding domain. Solution operators are often Lipschitz for many linear (and sometimes nonlinear) PDEs, with respect to the appropriate corresponding norms, satisfying:
\begin{align*}
    \norm{\mathcal{S}_{F_n}^t(u)-\mathcal{S}_{F_n}^t(v)}_{Y_n}\leq C_{F_n^t}\norm{u_0-v_0}_{X_n}
\end{align*}

Thus, in higher dimension, if we embed $w_0=\phi_n(u_0)$, and define a slightly perturbed initial condition $\tilde{w_0}=w_0+\delta$, we have:
\begin{align*}
    \norm{\mathcal{S}_{F_{n+1}}^t(w_0)-\mathcal{S}_{F_{n+1}}^t(\tilde w_0)}_{Y_{n+1}}\leq C_{F_{n+1}^t}\norm{\delta}_{X_{n+1}}.
\end{align*}
Then, since $\mathcal{S}_{F_{n+1}}^t(w_0)=\phi_n(\mathcal{S}_{F_n}^t(u_0))$, we have that integrating the "compatible part" of a perturbed initial condition and subsequently lifting remains a good approximation for the full perturbed solution. We note that many of the quantities here are dependent on the PDE and corresponding solution operator, and the norms over the corresponding Banach spaces. In particular the constants $C_{F_{n}^t}$ can be dimension- and time-dependent, possibly resulting in vacuous bounds from a practical approximation perspective. At the same time, they can be more forgiving, and force small perturbations to become smaller over time, thus yielding a robust approximation result.

\subsection{Compatibility of Tensor Image Functions}\label{ap:compat_tensor_image_functions}

Now we prove \Cref{thm:geometric_convolution_compat}, compatibility for geometric convolutions.

\begin{theorem*}[Restatement of \ref{thm:geometric_convolution_compat}]
    Let $C \in \mathcal{A}_{M,n,k+k',p\,p'}$ and $C' \in \mathcal{A}_{M,n+1,k+k',p\,p'}$ be convolution filters.
    Let the functions $f_n:\mathcal{A}_{N,n,k,p} \to \mathcal{A}_{N,n,k',p'}$ and $f_{n+1}:\mathcal{A}_{N,n+1,k,p} \to \mathcal{A}_{N,n+1,k',p'}$ be defined respectively as $f_n(A) = A*C$ for all $A \in \mathcal{A}_{N,n,k,p}$ and $f_{n+1}(A')=A'*C'$ for all $A' \in \mathcal{A}_{N,n+1,k,p}$.
    Then $f_n$ are compatible with respect to \eqref{eq:tensor_image_compat} if
    \begin{equation}
         C(\bar\imath_n) = \sum_{\bar\imath_{n+1} \in \pi_n^{-1}\qty(\bar\imath_n)} \pi_n\qty(C'\qty(\bar\imath_{n+1})) ~,
    \end{equation}
    for all $\bar\imath_n \in \ZNZ{N}{n}$, where $\pi_n^{-1}(\bar\imath_n)$ is the pre-image of the tensor projection to $\bar\imath_n$.
\end{theorem*}

\begin{proof}
    Let $C,C',f_n,f_{n+1}$ be defined as in the statement of the theorem.
    Suppose that $C(\bar\imath_n) = \sum_{\bar\imath_{n+1} \in \pi_n^{-1}\qty(\bar\imath_n)} \pi_n\qty(C'\qty(\bar\imath_{n+1}))$ for all $\bar\imath_n \in \ZNZ{N}{n}$.
    Thus,
    \begin{align}
        &\phi_n\qty(f_n(A))\qty(\bar\imath_{n+1}) \\
        &= \tilde{\phi}_n\qty(f_n(A)\qty(\pi_n(\bar\imath_{n+1}))) \\
        &= \tilde{\phi}_n\qty(\sum_{\bar a_n \in \ZNZ{M}{n}} \contract{k}{A\qty(\pi_n(\bar\imath_{n+1}) - c(\bar a_n)) \otimes C(\bar a_n)}) \\
        &= \sum_{\bar a_n \in \ZNZ{M}{n}} \tilde{\phi}_n\qty(\contract{k}{A\qty(\pi_n(\bar\imath_{n+1}) - c(\bar a_n)) \otimes C(\bar a_n)}) \\
        &= \sum_{\bar a_n \in \ZNZ{M}{n}} \tilde{\phi}_n \qty(\contract{k}{A\qty(\pi_n(\bar\imath_{n+1}) - \pi_n\qty(c(\bar a_{n+1}))) \otimes \sum_{\bar a_{n+1} \in \pi_n^{-1}(\bar a_n)} \pi_n\qty(C'(\bar a_{n+1}))}) \label{eq:apply_filter_compat} \\
        &= \sum_{\bar a_n \in \ZNZ{M}{n}} \tilde{\phi}_n \qty(\contract{k}{\pi_n\qty(\phi_n(A)\qty(\bar\imath_{n+1} - c(\bar a_{n+1}))) \otimes \sum_{\bar a_{n+1} \in \pi_n^{-1}(\bar a_n)} \pi_n\qty(C'(\bar a_{n+1}))}) \\
        &= \sum_{\bar a_n \in \ZNZ{M}{n}} \sum_{\bar a_{n+1} \in \pi_n^{-1}(\bar a_n)} \tilde{\phi}_n \qty(\contract{k}{\pi_n\qty(\phi_n(A)\qty(\bar\imath_{n+1} - c(\bar a_{n+1}))) \otimes  \pi_n\qty(C'(\bar a_{n+1}))}) \\
        &= \sum_{\bar a_{n+1} \in \ZNZ{M}{n}} \contract{k}{\tilde{\phi}_n \qty(\pi_n\qty(\phi_n(A)\qty(\bar\imath_{n+1} - c(\bar a_{n+1})) \otimes  C'(\bar a_{n+1})))} \label{eq:commute_projection_pseudoinverse} \\
        &= \sum_{\bar a_{n+1} \in \ZNZ{M}{n}} \contract{k}{\phi_n(A)\qty(\bar\imath_{n+1} - c(\bar a_{n+1})) \otimes  C'(\bar a_{n+1})} \label{eq:remove_psi_psi_pinv} \\
        &= f_{n+1}\qty(\phi_n(A))\qty(\bar\imath_{n+1})
    \end{align}
    In \eqref{eq:apply_filter_compat} we substitute $C(\bar a_n)$ based on our assumption, and additionally say $c(\bar a_n) = \pi_n(c(\bar a_{n+1})$, where $a_{n+1}$ is any pixel index in the pre-image $\pi_n^{-1}(\bar a_n)$.
    In \eqref{eq:commute_projection_pseudoinverse} we commute $\tilde{\phi}_n$ with $\contract{k}{\cdot}$ since the pseudoinverse fills the extra indices that would be added with zeros.
    Finally, in \eqref{eq:remove_psi_psi_pinv} we can remove $\tilde{\phi}_n(\pi_n(\cdot))$ because the embedded image $\phi_n(A)$ itself undergoes the projection pseudoinverse, so in effect we are saying that $\tilde{\phi}_n(\pi_n(\tilde{\phi}_n(a))) = \tilde{\phi}_n(a)$ which is true.
    The embedding and the function commute, so $f_n,f_{n+1}$ are compatible with \eqref{eq:tensor_image_compat}, and this completes the proof.
\end{proof}

The following is a simple corollary of \Cref{thm:geometric_convolution_compat} for a linear combination of filters.

\begin{corollary}[Compatibility for a linear combination of filters]\label{cor:linear_combination_of_filters_compat}
    Let $C_0,C_1,\ldots,C_{L_n} \in \mathcal{A}_{M,n,k+k',p\,p}$ be a set of $n$-dimensional convolution filters with associated coefficients $\alpha_0,\alpha_1,\ldots,\alpha_{L_n} \in \bbR$.
    Let $C'_0,C'_1,\ldots,C'_{L_{n+1}} \in \mathcal{A}_{M,n+1,k+k',p\,p'}$ be a set of $(n+1)$-dimensional convolution filters with associated coefficients $\alpha'_0,\ldots,\alpha'_{L_{n+1}} \in \bbR$.
    Let the functions $f_n:\mathcal{A}_{N,n,k,p} \to \mathcal{A}_{N,n,k',p'}$ and $f_{n+1}:\mathcal{A}_{N,n+1,k,p} \to \mathcal{A}_{N,n+1,k',p'}$ be defined respectively as $f_n(A) = A*\qty(\sum_{\ell=0}^{L_n} \alpha_\ell C_\ell)$ and $f_{n+1}(A)=A*\qty(\sum_{\ell=0}^{L_{n+1}} \alpha'_\ell C'_\ell)$ for all $A \in \mathcal{A}_{N,n,k,p}$.
    Then $f_n,f_{n+1}$ are compatible with \eqref{eq:tensor_image_compat} if
    \begin{equation}\label{eq:filter_scaling}
        \sum_{\ell = 0}^{L_n} \alpha_\ell C_\ell(\bar\imath_n) = \sum_{\bar\imath_{n+1} \in \pi_n^{-1}\qty(\bar\imath_n)} \sum_{\ell=0}^{L_{n+1}} \pi_n\qty(\alpha'_\ell C'_\ell \qty(\bar\imath_{n+1})) ~,
    \end{equation}
    for all $\bar\imath_n \in \ZNZ{N}{n}$, where $\pi_n^{-1}(\bar\imath_n)$ is the pre-image of the tensor projection to $\bar\imath_n$.
\end{corollary}

\begin{proof}
    This follows immediately from \Cref{thm:geometric_convolution_compat}.
\end{proof}

\begin{prop}[Compatibility of point-wise functions]\label{prop:pointwise_functions_compat}
    Let $f_n:\mathcal{A}_{N,n,k,p} \to \mathcal{A}_{N,n,k',p'}$ be a function defined point-wise, that is there exists some $h_n:\mathcal{T}_k\qty(\bbR^n,p)\to \mathcal{T}_{k'}\qty(\bbR^n,p')$ such that $f_n(A)(\bar\imath) = h_n(A(\bar\imath))$ for all $A \in \mathcal{A}_{N,n,k,p}$ and all $\bar\imath \in \ZNZ{N}{n}$.
    Then $f_n,f_{n+1}$ with associated $h_n,h_{n+1}$ are compatible with \eqref{eq:tensor_image_compat} if $h_{n+1}(\tilde{\phi}_{n}(a)) = \tilde{\phi}_n(h_n(a))$ for all $a \in \mathcal{T}_k\qty(\bbR^n,p)$.
\end{prop}

\begin{proof}
    Let $f_n,f_{n+1},h_n,h_{n+1}$ be as described in the statement of the proposition, and suppose that $h_{n+1}(\tilde{\phi}_{n}(a)) = \tilde{\phi}_n(h_n(a))$ for all $a \in \mathcal{T}_k\qty(\bbR^n,p)$.
    Let $A\in \mathcal{A}_{N,n+1,k,p}$ and let $\bar\imath_{n+1} \in \ZNZ{N}{n+1}$.
    Thus:
    \begin{align}
        f_{n+1}(\phi_n(A))\qty(\bar\imath_{n+1}) &= h_{n+1}(\phi_n(A)(\bar\imath_{n+1})) \\
        &= h_{n+1}(\tilde{\phi}_n(A(\pi_n(\bar\imath_{n+1})))) \\
        &= \tilde{\phi}_n(h_n(A(\pi_n(\bar\imath_{n+1})))) \\
        &= \tilde{\phi}_n(f_n(A)(\pi_n(\bar\imath_{n+1}))) \\
        &= \phi_n(f_n(A))(\bar\imath_{n+1})
    \end{align}
    This completes the proof.
\end{proof}

Now we would like to confirm that the Vector Neuron nonlinearity \cite{deng2021vectorneuronsgeneralframework,gi_net} used in our models satisfies the tensor embedding condition and is therefore compatible.

\begin{corollary}[Compatibility of Vector Neuron nonlinearity]\label{cor:vector_neuron_compat}
    The Vector Neuron nonlinearity (\Cref{def:vnn_for_tensors}) is compatible with \eqref{eq:tensor_image_compat}.
\end{corollary}

\begin{proof}
    The Vector Neuron nonlinearity is defined for (multiple channels of) an entire tensor image because the function is the same at each pixel.
    Thus it suffices to show that the definition on a single pixel commutes with $\tilde{\phi}_n$.
    Let $a_i \in \mcT_k\qty(\bbR^n,p)$ for $i=1,\ldots,r$ be the channels of tensors at a particular pixel index.
    If we consider $b = \sum_{i=1}^r \alpha_i a_i$ as a function of the $a_i$, then clearly by linearity 
    \begin{equation}
        b(\tilde{\phi}_n(a_i)) = \sum_{i=1}^r \alpha_i \tilde{\phi}_n (a_i) = \tilde{\phi}_n\qty(\sum_{i=1}^r \alpha_i a_i) = \tilde{\phi}_n(b(a_i)) ~.
    \end{equation}
    Also, $\norm{\tilde{\phi}_n(b)}_2 = \norm{b}_2$ because zero-padding does not change the tensor norm, so 
    \begin{equation}
        \hat{b}(\tilde{\phi}_n(a_i)) = \frac{b\qty(\tilde{\phi}_n(a_i))}{\norm{b\qty(\tilde{\phi}_n(a_i))}_2} = \frac{\tilde{\phi}_n(b(a_i))}{\norm{b(a_i)}_2} = \tilde{\phi}_n\qty(\frac{b(a_i)}{\norm{b(a_i)}_2}) = \tilde{\phi}_n\qty(\hat{b}(a_i)) ~.
    \end{equation}
    Thus, 
    \begin{equation}
        \contract{k}{\tilde{\phi}_n(a_i) \otimes \hat{b}(\tilde{\phi}_n(a_i))} = \contract{k}{\tilde{\phi}_n(a_i) \otimes \tilde{\phi}_n\qty(\hat{b}(a_i))} = \contract{k}{a_i \otimes \hat{b}(a_i)}
    \end{equation}
    Finally, we conclude that,
    \begin{align}
        \sigma\qty(\qty(\tilde{\phi}_n(a_j))_{j=1}^r) &= \tilde{\phi}_n(a_i) - \tilde{\phi}_n(a_{i,\parallel}) + \frac{\tilde{\sigma}\qty(\contract{k}{\tilde{\phi}_n(a_i) \otimes \hat{b}\qty(\tilde{\phi}_n(a_i))})}{\qty|\contract{k}{\tilde{\phi}_n(a_i) \otimes \hat{b}\qty(\tilde{\phi}_n(a_i))}|}\tilde{\phi}_n(a_{i,\parallel}) \\
        &= \tilde{\phi}_n(a_i) - \tilde{\phi}_n(a_{i,\parallel}) + \frac{\tilde{\sigma}\qty(\contract{k}{a_i \otimes \hat{b}(a_i)})}{\qty|\contract{k}{a_i \otimes \hat{b}(a_i)}|}\tilde{\phi}_n(a_{i,\parallel}) \\
        &= \tilde{\phi}_n\qty(a_i -a_{i,\parallel} + \frac{\tilde{\sigma}\qty(\contract{k}{a_i \otimes \hat{b}(a_i)})}{\qty|\contract{k}{a_i \otimes \hat{b}(a_i)}|}a_{i,\parallel}) \\
        &= \tilde{\phi}_n\qty(\sigma\qty((a_j)_{j=1}^r)) ~.
    \end{align}
    Thus the single pixel function commutes with $\tilde{\phi}_n,$ so the Vector neuron nonlinearity is compatible.
\end{proof}

Next we will show that max pooling is compatible.
\begin{prop}
    The max pool layer (\Cref{def:max_pool}) is compatible with \eqref{eq:tensor_image_compat}.
\end{prop}

\begin{proof}
    Let $A \in \mathcal{A}_{N,n,k,p}$, let $b$ be an integer that divides $N$, and let $\bar\imath_{n+1} \in \ZNZ{N}{n+1}$.
    Then 
    \begin{align}
        &\phi_n(\text{max\,pool}(A,b))(\bar\imath_{n+1}) \\
        &= \tilde{\phi}_n\qty(\text{max\,pool} (A,b)\qty(\pi_n\qty(\bar\imath_{n+1}))) \\
        &= \tilde{\phi}_n\qty(A\qty(b\,\pi_n \qty(\bar\imath_{n+1}) + \argmax_{\bar a \in \ZNZ{b}{n}}\norm{A\qty(b\, \pi_n\qty(\bar\imath_{n+1}) + \bar a)}_2)) \\
        &= \tilde{\phi}_n\qty(A\qty(b\,\pi_n \qty(\bar\imath_{n+1}) + \pi_n\qty(\argmax_{\bar a \in \ZNZ{b}{n+1}}\norm{A\qty(b\, \pi_n\qty(\bar\imath_{n+1}) + \pi_n(\bar a))}_2))) \label{eq:max_pool_projection} \\
        &= \tilde{\phi}_n\qty(A\qty(\pi_n \qty(b\,\bar\imath_{n+1} + \argmax_{\bar a \in \ZNZ{b}{n+1}}\norm{\tilde{\phi}_n\qty(A\qty(\pi_n\qty(b\,\bar\imath_{n+1} + \bar a)))}_2))) \\
        &= \phi_n(A)\qty(b\,\bar\imath_{n+1} + \argmax_{\bar a \in \ZNZ{b}{n+1}}\norm{\phi_n(A)\qty(b\,\bar\imath_{n+1} + \bar a)}_2) \\
        &= \text{max\,pool}(\phi_n(A),b)(\bar\imath_{n+1}) ~.
    \end{align}
    The key insight of this proof is line \eqref{eq:max_pool_projection} where we note that the projection $\pi_n$ of the $\argmax$ over the $n+1$-dimensional hypercube of side length $b$ pixels is equivalent to the $\argmax$ over the $n$-dimensional hypercube.
    This completes the proof.
\end{proof}


\section{Lifting filters}\label{app.lifting_filters}

We first define the zero-pad embedding.

\begin{defn}[Zero-pad embedding]\label{def:zero_pad_embedding}
    The \emph{zero-pad embedding}, denoted $\varphi_n: \mathcal{A}_{M,n,k,p} \to \mathcal{A}_{M,n+1,k,p}$, is defined for all images $C \in \mathcal{A}_{M,n,k,p}$ and all $\bar\imath_{n+1} \in \ZNZ{M}{n+1}$ as
    \begin{equation}
        \varphi_n(C)(\bar\imath_{n+1}) = \begin{cases}
            \tilde{\phi}_n(C(\pi_n(\bar\imath_{n+1}))) & \text{ for } \tilde{\phi}_n(\pi_n(\bar\imath_{n+1})) = \bar\imath_{n+1} \\
            0 & \text{ otherwise }
        \end{cases} ~.
    \end{equation}
    Note that for filters, the indices are centered so that $\bar\imath = \qty{0,\ldots,0}$ is the center pixel.
\end{defn}

\begin{prop*}[Restatement of \Cref{prop:simple_argmin_objective}]
    Let filters and coefficients be as defined in \Cref{def:warmstart_coefficients}.
    If we let $\mathcal{A}'$ be the distribution of scalar images with independent identically distributed (iid) pixels with $\underset{A' \sim \mathcal{A}'}{E}\qty[A'(\bar\imath)] = 0$ and $\underset{A' \sim \mathcal{A}'}{E}\qty[A'(\bar\imath)^2] = 1$, then \eqref{eq:warmstart_coefficients} is equivalent to $\mathcal{L}\qty(\pi_n\qty(\varphi_n(C) - C'))$.
\end{prop*}

\begin{proof}
    Let filters and coefficients be as defined in \Cref{def:warmstart_coefficients}.
    Thus
    \begin{align}
        &\underset{h \in G_{n+1}}{E}\qty[\underset{A' \sim \mathcal{A}'}{E}\qty[ \mathcal{L}\qty(\qty(\rho_n\qty(h \cdot A') \ast C) - \rho_n\qty(h \cdot \qty(A' \ast C' )))]] \\
        &=\underset{h \in G_{n+1}}{E}\qty[\underset{A' \sim \mathcal{A}'}{E}\qty[ \mathcal{L}\qty(\qty(\rho_n\qty(h \cdot A') \ast C) - \rho_n\qty(\qty(h \cdot A') \ast C' ))]] \\
        &=\underset{A' \sim \mathcal{A}'}{E}\qty[ \mathcal{L}\qty(\qty(\rho_n \qty(A') \ast C) - \rho_n\qty(A' \ast C' ))] \\
        &=\underset{A' \sim \mathcal{A}'}{E}\qty[ \mathcal{L}\qty(\rho_n \qty(A' \ast \varphi_n(C)) - \rho_n\qty(A' \ast C'))] \\
        &=\underset{A' \sim \mathcal{A}'}{E}\qty[ \mathcal{L}\qty(\rho_n \qty(A' \ast \qty(\varphi_n(C) - C')))] ~.
    \end{align}
    Now we further break down the loss function $\mathcal{L}$:
    \begin{align}
        &\underset{A' \sim \mathcal{A}'}{E}\qty[ \mathcal{L}\qty(\rho_n \qty(A' \ast \qty(\varphi_n(C) - C')))] \\
        &= \underset{A' \sim \mathcal{A}'}{E}\qty[\frac{1}{N^n} \sum_{\bar\imath_n \in \ZNZ{N}{n}}\norm{\rho_n \qty(A' \ast \qty(\varphi_n(C) - C'))(\bar\imath_n)}^2_2] \\
        &= \underset{A' \sim \mathcal{A}'}{E}\qty[ \frac{1}{N^n} \sum_{\bar\imath_n \in \ZNZ{N}{n}}\norm{\rho_n \qty(A' \ast \qty(\varphi_n(C) - C'))(\bar\imath_n)}^2_2] \\
        &= \underset{A' \sim \mathcal{A}'}{E}\qty[ \frac{1}{N^n} \sum_{\bar\imath_d \in \ZNZ{N}{n}}\norm{\pi_n\qty(\qty(A' \ast \qty(\varphi_n(C) - C'))\qty(\tilde{\phi}_n\qty(\bar\imath_n)))}^2_2] \\
        &= \underset{A' \sim \mathcal{A}'}{E}\qty[ \frac{1}{N^n} \sum_{\bar\imath_n \in \ZNZ{N}{n}}\norm{\pi_n\qty(\sum_{\bar a_{n+1} \in \ZNZ{M}{n+1}}A'\qty(\tilde{\phi}_n\qty(\bar\imath_n) - \bar a_{n+1}) \otimes \qty(\varphi_n(C) - C')(\bar a_{n+1}))}^2_2] \\
        &= \underset{A' \sim \mathcal{A}'}{E}\qty[ \frac{1}{N^n} \sum_{\bar\imath_n \in \ZNZ{N}{n}}\norm{\sum_{\bar a_{n+1} \in \ZNZ{M}{n+1}}\pi_n\qty(A'\qty(\tilde{\phi}_n\qty(\bar\imath_n) - \bar a_{n+1})) \otimes \pi_n\qty(\qty(\varphi_n(C) - C')(\bar a_{n+1}))}^2_2] \label{eq:opt_minimization_left_off} ~.
    \end{align}
    Now we note that the squared norm of a sum of vectors can be decomposed as two sums, the first given by a norm on the matching vectors, and the second a double sum over the inner product of the different vectors.
    In our setting, the norm is still the tensor norm, and the inner product is the full k-contraction.
    Thus, the first sum is,
    \begin{align}
        &\sum_{\bar a_{n+1} \in \ZNZ{M}{n+1}} \norm{\pi_n\qty(A'\qty(\tilde{\phi}_n\qty(\bar\imath_n) - \bar a_{n+1})) \otimes \pi_n\qty(\qty(\varphi_n(C) - C')(\bar a_{n+1}))}^2_2 \\
        &=\sum_{\bar a_{n+1} \in \ZNZ{M}{n+1}} A'\qty(\tilde{\phi}_n\qty(\bar\imath_n) - \bar a_{n+1})^2 \norm{\pi_n\qty(\qty(\varphi_n(C) - C')(\bar a_{n+1}))}^2_2 ~,
    \end{align}
    with the equality following from the assumption that $A'$ is a scalar image.
    Likewise, the second sum, after moving the scalars outside the contraction, is
    \begin{align}
        &=2\sum_{\bar a_{n+1} \not = \bar b_{n+1} \in \ZNZ{M}{n+1}} A'\qty(\tilde{\phi}_n\qty(\bar\imath_n) - \bar a_{n+1}) A'\qty(\tilde{\phi}_n\qty(\bar\imath_n) - \bar b_{n+1}) \\
        &\hspace{4cm}\contract{k}{\pi_n\qty(\qty(\varphi_n(C) - C')(\bar a_{n+1})) \otimes \pi_n\qty(\qty(\varphi_n(C) - C')(\bar b_{n+1}))}
    \end{align}
    By linearity we can now directly apply the expectation to the terms with $A'$.
    Since $\underset{A' \sim \mathcal{A}'}{E}\qty[A'\qty(\tilde{\phi}_n\qty(\bar\imath_n) - \bar a_{n+1})] = 0$ by assumption, the second sum is $0$.
    By assumption $\underset{A' \sim \mathcal{A}'}{E}\qty[A'\qty(\tilde{\phi}_n\qty(\bar\imath_n) - \bar a_{n+1})^2] = 1$ in the first sum.
    If we plug these results back into \eqref{eq:opt_minimization_left_off} we have:
    \begin{align}
        & \underset{A' \sim \mathcal{A}'}{E}\qty[ \frac{1}{N^n} \sum_{\bar\imath_n \in \ZNZ{N}{n}}\norm{\sum_{\bar a_{n+1} \in \ZNZ{M}{n+1}}\pi_n\qty(A'\qty(\tilde{\phi}_n\qty(\bar\imath_n) - \bar a_{n+1})) \otimes \pi_n\qty(\qty(\varphi_n(C) - C')(\bar a_{n+1}))}^2_2] \\
        &=\frac{1}{N^n} \sum_{\bar\imath_n \in \ZNZ{N}{n}} \sum_{\bar a_{n+1} \in \ZNZ{M}{n+1}} \norm{\pi_n\qty(\qty(\varphi_n(C) - C')(\bar a_{n+1}))}^2_2 \\
        &= \sum_{\bar a_{n+1} \in \ZNZ{M}{n+1}} \norm{\pi_n\qty(\qty(\varphi_n(C) - C')(\bar a_{n+1}))}^2_2 \\
        &= \mathcal{L}\qty(\pi_n\qty(\varphi_n(C) - C')) ~.
    \end{align}
    This completes the proof.
\end{proof}

We now show the coefficients derived from \Cref{prop:simple_argmin_objective}.

\begin{table}[h]
    \centering
    \begin{tabular}{c|cccc}
        Type & Dimension & Weights \\
        \toprule
        Scalar filters 1D & $\alpha_0$ & $\alpha_1$ & \\
        2D Nepo & $\alpha_0 - \frac{2}{3}\alpha_1$ & $\frac{1}{3}\alpha_1$ & $\frac{1}{3}\alpha_1$ \\
        2D Copy & $\alpha_0 + 2 \alpha_1$ & $-\alpha_1$ & $\alpha_1$ \\
        2D Zeros & $\alpha_0 - 2 \alpha_1$ & $\alpha_1$ & 0 \\
        \midrule
        Scalar filters 2D & $\alpha_0$ & $\alpha_1$ & $\alpha_2$ \\
        3D & $\alpha_0 - \frac{10}{9}\alpha_1 + \frac{4}{9}\alpha_2$ & $\frac{5}{9}\alpha_1 - \frac{2}{9}\alpha_2$ & $\frac{2}{9}\alpha_1 + \frac{1}{9}\alpha_2$ & $-\frac{1}{9}\alpha_1 + \frac{4}{9}\alpha_2$ \\
        \midrule
        Vector filters 2D & $\alpha_0$ & $\alpha_1$ \\
        3D & $\alpha_0 - \frac{2}{3}\alpha_1$ & $\frac{1}{3}\alpha_1$ & $\frac{1}{3}\alpha_1$ \\
        \midrule
        2-tensor filters 2D & $\alpha_0$ & $\alpha_1$ & $\alpha_2$ & $\alpha_3$ \\
        3D & $\alpha_0 - \frac{2}{3}\alpha_1 + \frac{2}{3}\alpha_3$ & $\frac{5}{9}\alpha_1 - \frac{2}{9}\alpha_2 + \frac{1}{9}\alpha_3$ & $\frac{1}{9}\alpha_1 + \frac{2}{9}\alpha_2 - \frac{1}{9}\alpha_3$ & $\frac{1}{3}\alpha_2$ \\
        \cmidrule{2-5}
        2D & $\alpha_4$ \\
        3D & $\alpha_4$ & 0 & $-\frac{4}{9}\alpha_1 + \frac{4}{9}\alpha_2 - \frac{8}{9}\alpha_3$ & $\frac{2}{9}\alpha_1 - \frac{2}{9}\alpha_2 - \frac{2}{9}\alpha_3$
    \end{tabular}
    \caption{
    Coefficients for positive parity filters with side length 3.
    For each set of two, the first row is the weights for the lower dimensional filters, in order from the center of the filter outwards.
    The second row is the weights for the higher dimensional filters in terms of the weights of the lower dimensional filters.
    The 2D Copy and 2D Zeros coefficient transfer strategies are explained in \Cref{ap:warmstart_ablation}.
    }
    \label{tab:conversion_coefficients}
\end{table}

\section{Experimental Details}\label{ap:experimental_details}

\subsection{Data}\label{ap:data}

We use data from three PDEs, the heat equation (left), Burgers' equation (right):
\begin{equation}
    \frac{du}{dt} = \Delta u ~, \hspace{2cm} \frac{du}{dt} = -b \frac{1}{2} \nabla \cdot (u \otimes u) + \nu \nabla^2 u ~,
\end{equation}
and the compressible Navier--Stokes equations:
\begin{align}
    &\partial_t \rho + \nabla \cdot (\rho \vec{v}) = 0, \hspace{3mm} \rho\qty(\partial_t \vec{v} + \vec{v} \cdot \nabla \vec{v}) = - \nabla p + \eta \Delta \vec{v} + \qty(\zeta + \eta / 3)\nabla(\nabla \cdot \vec{v}), \\
    & \partial_t \qty(\epsilon + \rho v^2/2)+\nabla \cdot \qty[\qty(p + \epsilon + \rho v^2/2)\vec{v}-\vec{v}\cdot \sigma'] = 0
\end{align}

These are equations are used to generate fields which we discretize to images.
All images had a side length of 64 pixels.
The initial heat data $u_0$ is generated with a uniform distribution from $-\sqrt{3}$ to $\sqrt{3}$ in each pixel, and then the value at time $t=1$ was calculated using the analytic heat kernel with all points, where the distance to each point was the closest distance on the torus.

The data for Burgers' equation was generated by the Apebench package \cite{koehler2024apebench} using the difficulty scenario. The diffusion coefficient for 2D was 0.75 and for 3D was 2.25, while the convection coefficient was -0.75 for 2D and -2.25 for 3D. 
This was to undue the automatic scaling of the coefficients by dimension in the difficulty scenario so that the same parameters were used across dimensions.
Each trajectory was run for 50 time steps with 0 warmup steps.
Each trajectory of 50 steps was then split into 49 input output pairs of one timestep as part of the training data.

The data for the compressible Navier--Stokes equations is generated by PDEBench \cite{takamoto2024pdebenchextensivebenchmarkscientific}.
Although PDEBench has many datasets including Navier--Stokes in 2D and 3D, we require that exactly the same hyperparameters are used to generate data across dimensions.
Thus we used their code to generate additional data with shear and bulk viscosity of 1e-08, and Mach number 0.1.
Each trajectory was run for 21 time steps which where split into 17 data points of 4 input steps and 1 output step.

The size of the datasets in number of float values for $n$ dimensions, $b$ trajectories, $t$ time steps, image side length $N$, $s$ scalar channels, and $v$ vector channels is $btN^n(s+vn)$.

\subsection{Model and Training}\label{ap:training}

We used a UNet \cite{ronneberger2015unetconvolutionalnetworksbiomedical} as that was the best performing model in \cite{gi_net}.
We used filter side lengths of 3 with four downsampling blocks by max pooling and two convolution layers between each downsample.
For nonlinearities we used the Vector Neuron \cite{deng2021vectorneuronsgeneralframework} described in \cite{gi_net} with GeLU nonlinearity.
We used 48 channels as the convolution default which doubles every downsample.

We used two metrics during training and evaluation.
The first is the sum of mean squared error (SMSE) which takes the mean squared error over each field, then takes the sum over the fields.
If $A_j$ is the true \tensor{k}{p} image for field $j$ and $\hat{A}_j$ is our predicted \tensor{k}{p} image for field $j$, then the $\mathcal{L}_{\text{SMSE}}$ is defined as,
\begin{equation}
    \mathcal{L}_{\text{SMSE}}\qty(\qty{A_j}_{j=1}^J, \qty{\hat{A}_j}_{j=1}^J) = \sum_j \frac{1}{N^n} \sum_{\bar\imath} \norm{A_j(\bar\imath) - \hat{A}_j(\bar\imath)}_2^2 ~,
\end{equation}
where $\norm{\cdot}_2$ is the tensor norm.
The second metric is the relative error which is defined as,
\begin{equation}
    \mathcal{L}_{\text{rel}}\qty(\qty{A_j}_{j=1}^J, \qty{\hat{A}_j}_{j=1}^J) = \frac{1}{J}\sum_j \frac{\sqrt{\sum_{\bar\imath} \norm{A_j(\bar\imath) - \hat{A}_j\qty(\bar\imath)}_2^2}}{\sqrt{\sum_{\bar\imath} \norm{A_j(\bar\imath)}_2^2} + \epsilon} * 100 ~,
\end{equation}
where $\epsilon$ is some fixed small value to prevent division by 0 numerical issues, and we multiply by 100 to get the error as a percent.
We set $\epsilon = \verb|1e-5|$.
We found that calculating the relative error per pixel led to numerical instability during training, so we use the above version which can be viewed as taking the norm over the entire image for a field.
We train all models with the relative error, and report results for the relative error and SMSE.

We followed a similar training regime as in \cite{gi_net}.
We train for 50 epochs using the AdamW optimizer \cite{loshchilov2018adamw} with a weight decay of \verb|1e-5| and a cosine decay schedule \cite{loshchilov2017cosinedecay} to \verb|1e-7| with 5 epochs of warmup from \verb|1e-8|.
Peak learning rates were tuned for each problem for each dimension and for each different number of tuning points.
Learning rates were chosen from the set $\qty{1e-5, 5e-5, 1e-4, 5e-4, 1e-3, 5e-3}$ based on performance on a validation set of 128 points for the heat equation, 8 trajectories for Burgers' equation, or 32 trajectories for Navier--Stokes.
For the ablation over the alternative coefficient transfer strategies, the same learning rates tuned for our warm-start strategy were used for all the strategies.
It is possible that the performance could be improved by individually tuning the learning rates for each strategy, but the results for 0 tuning points shows that our warm-start strategy is likely to remain the best.
The resulting learning rates are shown in \Cref{tab:learning_rates}.

\begin{table}[h]
    \centering
    \begin{tabular}{ccc|ccccc}
        PDE & Models & Params & 1 & 4 & 8 & 32 & 128 \\
        \toprule
        & 1D Pre-train & 3,926,737 & & & & & \verb|1e-3| \\
        Heat & 2D Baseline & 5,691,697 & \verb|5e-4| & \verb|5e-4| & \verb|5e-4| & \verb|5e-4| & \verb|5e-4| \\ 
        & 2D Fine-tune & 5,691,697 & \verb|5e-4| & \verb|5e-4| & \verb|5e-4| & \verb|5e-4| & \verb|5e-4| \\ 
        \midrule
        & 2D Pre-train & 11,576,305 & & & \verb|1e-4| & & \\
        Burgers & 3D Baseline & 16,871,185 & \verb|1e-4| & \verb|1e-4| & \verb|1e-4| & & \\ 
        & 3D Fine-tune & 16,871,185 & \verb|1e-4| & \verb|1e-4| & \verb|1e-4| & & \\ 
        \midrule
        & 2D Pre-train & 25,114,131 & & & & & \verb|1e-4| \\
        N-S & 3D Baseline & 35,705,235 & \verb|5e-4| & \verb|1e-4| & \verb|1e-4| & \verb|1e-4| & \\ 
        & 3D Fine-tune & 35,705,235 & \verb|1e-4| & \verb|1e-4| & \verb|1e-4| & \verb|5e-5| & \\ 
    \end{tabular}
    \caption{Pre-training and fine-tuning learning rates for the number of data points or trajectories.}
    \label{tab:learning_rates}
\end{table}

We trained on a single RTX 6000 Ada GPU with a batch size of 8 for all experiments except for the Burgers' and Navier--Stokes 3D experiments where we used 2 GPUs with a batch size of 2 per GPU for an effective batch size of 4.

\subsection{Warm-start Ablation}\label{ap:warmstart_ablation}

We explored whether the warm-start coefficient transfer, \Cref{def:warmstart_coefficients} and \Cref{prop:simple_argmin_objective}, is the correct equation to optimize.
We experimented with two reasonable alternative strategies that satisfied compatibility but made different choices with the free parameter.
In the Copy strategy, we set the new coefficient for the outermost weight equal to the old coefficient for the outermost weight, then adjusted the other weights to maintain compatibility.
In the Zeros strategy, we set the new weight to 0 and adjusted the others likewise.
We tried these strategies for the heat equation, with the results shown in \Cref{fig:warmstart_ablation} and \Cref{tab:warmstart_ablation}.
Clearly our warm-start strategy performs better.
It is also interesting to note that in some cases, and particularly for the Copy strategy, a naive weight transfer strategy is worse than no pre-training at all.
This is true even though all strategies satisfy compatibility.

\begin{figure}[h]
    \centering
    \includegraphics[width=0.5\textwidth]{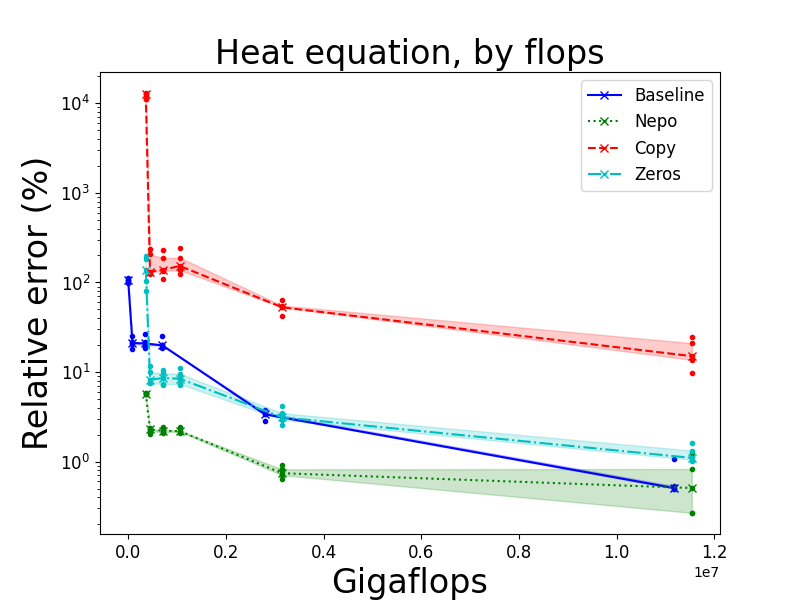}
    \caption{
    The median relative error performance of each warm-start strategy with the heat equation over varying numbers of fine tuning points, indexed by the total number of gigaflops to pre-train and fine-tune the models.
    Note the order of magnitude of the x-tick values, and the log values of the y-axis.
    Each experiment was run 5 times, with the lines representing the median value, the shaded region representing the interquartile range, and a dot for each value.
    }
    \label{fig:warmstart_ablation}
\end{figure}

\begin{table}[h]
    \centering
    \begin{tabular}{cc|cccccc}
        PDE & Models & 0 & 1 & 4 & 8 & 32 & 128 \\
        \toprule
        & Baseline & 105.3\% & 21.0\% & 20.8\% & 19.8\% & 3.3\% & \textbf{0.5\%} \\
        Heat & Copy & 12,797.3\% & 131.6\% & 138.5\% & 152.5\% & 52.7\% & 15.0\% \\ 
        & Zeros & 138.3\% & 8.1\% & 8.1\% & 8.4\% & 3.1\% & 1.0\% \\
        & Nepo & \textbf{5.7\%} & \textbf{2.2\%} & \textbf{2.2\%} & \textbf{2.1\%} & \textbf{0.7\%} & \textbf{0.5}\% \\
    \end{tabular}
    \caption{The median relative error performance over 5 trials of each warm-start strategy with the heat equation over varying numbers of fine tuning points.
    The column labels represent the number of fine-tuning points.
    With higher precision, the baseline method does better than Nepo by $0.001$\%.
    }
    \label{tab:warmstart_ablation}
\end{table}

\subsection{Additional Results}\label{ap:additional_results}

\begin{figure}[h]
    \centering
    \begin{minipage}{0.32\textwidth}
        \includegraphics[width=\textwidth]{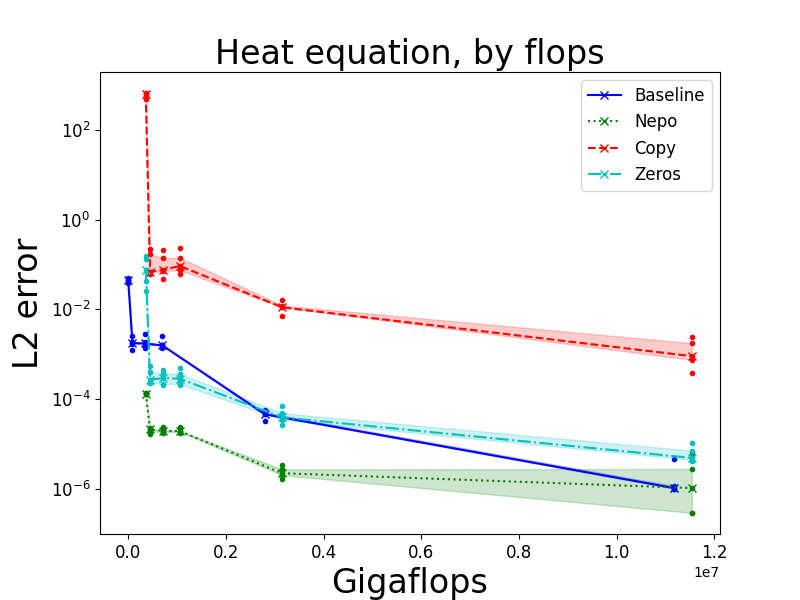}
    \end{minipage}
    \begin{minipage}{0.32\textwidth}
        \includegraphics[width=\textwidth]{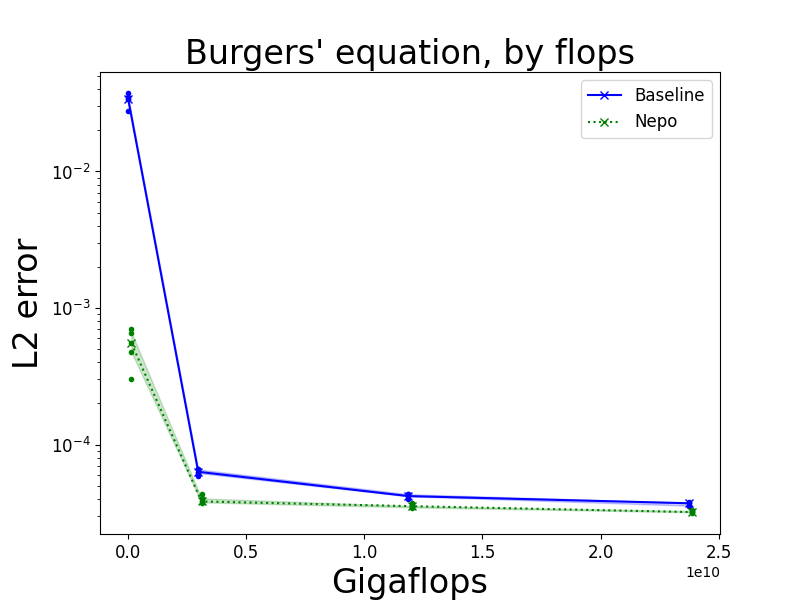}
    \end{minipage}
    \begin{minipage}{0.32\textwidth}
        \includegraphics[width=\textwidth]{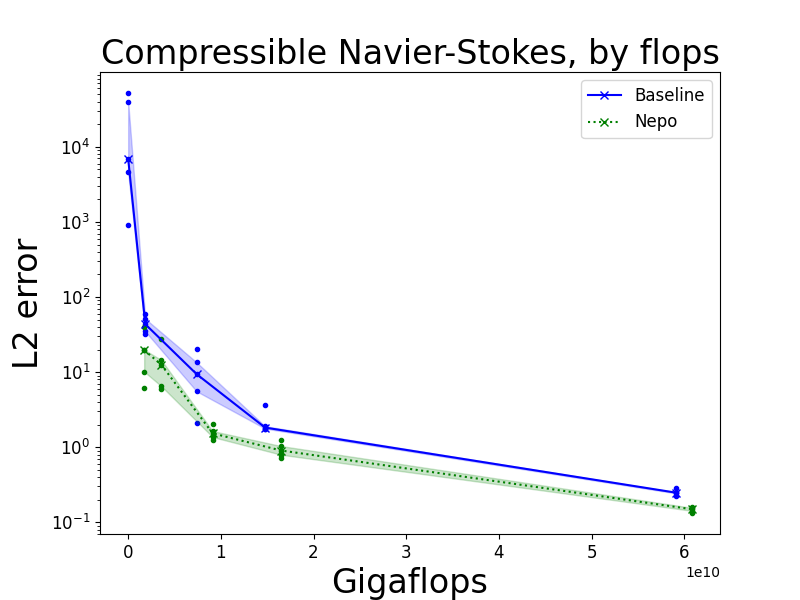}
    \end{minipage}
    \begin{minipage}{0.32\textwidth}
        \centering
        \textbf{(a)}
    \end{minipage}
    \begin{minipage}{0.32\textwidth}
        \centering
        \textbf{(b)}
    \end{minipage}
    \begin{minipage}{0.32\textwidth}
        \centering
        \textbf{(c)}
    \end{minipage}
    \caption{$\ell_2$ error performance of each PDE over varying numbers of fine tuning points, indexed by the total number of gigaflops to pre-train and fine-tune the models.
    Note the order of magnitude of the x-tick values, and the log values of the y-axis.
    Each experiment was run 5 times, with the lines representing the median value, the shaded region representing the interquartile range, and a dot for each value.
    \textbf{(a)} The heat equation trained on 0, 1, 4, 8, 32, and 128 fine-tuning points.
    \textbf{(b)} Burgers' equation trained on 0, 1, 4, and 8 trajectories.
    \textbf{(c)} The compressible Navier--Stokes equations trained on 0, 1, 4, 8, and 32 trajectories.
    }
    \label{fig:warm-start_plots_l2}
\end{figure}

\begin{table}[h]
    \centering
    \begin{tabular}{cc|cccccc}
        PDE & Models & 0 & 1 & 4 & 8 & 32 & 128 \\
        \toprule
        & Baseline & 4.4e-02 & 1.7e-03 &1.7e-03 & 1.5e-03 & 4.5e-05 & \textbf{1.0e-06} \\
        Heat & Copy & 6.5e+02 & 6.9e-02 & 7.6e-02 & 9.3e-02 & 1.1e-02 & 9.0e-04 \\ 
        & Zeros & 7.6e-02 & 2.6e-04 & 2.9e-04 & 2.8e-04 & 3.8e-05 & 4.8e-06 \\
        & Nepo & \textbf{1.3e-04} & \textbf{2.0e-05} & \textbf{1.9e-05} & \textbf{1.9e-05} & \textbf{2.2e-06} & \textbf{1.0e-06} \\
        \midrule
        Burgers & Baseline & 3.3e-02 & 6.3e-05 & 4.1e-05 & 3.7e-05 & Time out & Time out \\
        & Nepo & \textbf{5.5e-04} & \textbf{3.8e-05} & \textbf{3.5e-05} & \textbf{3.2e-05} & Time out & Time out \\
        \midrule
        N-S & Baseline & 6.9e+03 & 4.3e+01 & 9.3e+00 & 1.8e+00 & 2.4e-01 & Time out \\
        & Nepo & \textbf{1.9e+01} & \textbf{1.2e+01} & \textbf{1.5e+00} & \textbf{9.0e-01} & \textbf{1.4e-01} & Time out \\
    \end{tabular}
    \caption{The median $\ell_2$ error performance over 5 trials of each PDE over varying numbers of fine tuning points.
    N-S refers to compressible Navier--Stokes.
    For the heat equation, the column labels represent the number of fine-tuning points, while for Burgers and Navier--Stokes, they refer to the number of trajectories of 17 and 49 time steps each, respectively.
    The positions in the table labeled ``Time out'' were training loops that would have taken longer than 15 hours per trial.
    With higher precision, the baseline method does better than Nepo for the heat equation by 1e-09.
    }
    \label{tab:smse}
\end{table}

\begin{table}[h]
    \centering
    \begin{tabular}{cc|cccccc}
        PDE & Models & n=0 & n=1 & n=4 & n=8 & n=32 & n=128 \\
        \toprule
        Heat & Baseline & 0 & 8.732e+04 & 3.493e+05 & 6.986e+05 & 2.794e+06 & 1.117e+07 \\
        & Nepo & 3.594e+05 & 4.467e+05 & 7.087e+05 & 1.058e+06 & 3.153e+06 & 1.153e+07 \\ 
        \midrule
        Burgers & Baseline & 0 & 2.965e+09 & 1.186e+10 & 2.372e+10 \\
        & Nepo & 1.396e+08 & 3.105e+09 & 1.200e+10 & 2.386e+10 \\
        \midrule
        N-S & Baseline & 0 & 1.846e+09 & 7.386e+09 & 1.477e+10 & 5.908e+10 \\
        & Nepo & 1.724e+09 & 3.571e+09 & 9.110e+09 & 1.649e+10 & 6.081e+10 \\
    \end{tabular}
    \caption{Estimated gigaflops of training process.
    Nepo warm-start includes the time to pre-train on lower dimensional data and convert the model.}
    \label{tab:flops}
\end{table}

\newpage




\end{document}